\documentclass{article}

\usepackage[preprint]{neurips_2026}

\usepackage[utf8]{inputenc} % allow utf-8 input
\usepackage[T1]{fontenc}    % use 8-bit T1 fonts
\usepackage{hyperref}       % hyperlinks
\usepackage{url}            % simple URL typesetting
\usepackage{booktabs}       % professional-quality tables
\usepackage{amsfonts}       % blackboard math symbols
\usepackage{nicefrac}       % compact symbols for 1/2, etc.
\usepackage{microtype}      % microtypography
\usepackage{xcolor}         % colors
\usepackage{graphicx}
\usepackage{algorithm}
\usepackage{algorithmic}
\usepackage{subcaption}

\usepackage{amsmath}
\usepackage{amssymb}
\usepackage{bbm}
\usepackage{mathtools}
\usepackage{amsthm}
\usepackage{thmtools}
\usepackage{thm-restate}\usepackage{pifont}
\usepackage[capitalize,noabbrev]{cleveref}
\usepackage{float}
\usepackage{multirow}

\theoremstyle{plain}

\theoremstyle{definition}

\theoremstyle{remark}

\newcommand{\cmark}{\ding{51}}%

\title{LOKI: Memory-Free Null-Space Constrained Lifelong Knowledge Editing}

\author{%
  Masih Eskandar, Miquel Sirera Perelló, Stratis Ioannidis, and Jennifer Dy %\thanks{Use footnote for providing further information
 \\
  Department of Electrical and Computer Engineering\\
  Northeastern University\\
  Boston, MA, 02115 \\
  \texttt{\{eskandar.m,sirera.m,ioannidis,j.dy\}@northeastern.edu} \\
}

\begin{document}

\maketitle

\begin{abstract}
  Lifelong knowledge editing aims to efficiently and sequentially update language models over time, as new knowledge becomes available or when the model makes mistakes, while preserving acceptable performance on past knowledge. One unresolved challenge is that existing methods modify a fixed set of layers for all new knowledge samples, reducing flexibility and increasing catastrophic forgetting. Another is requiring access to previous knowledge and extensive pre-processing to obtain data statistics. To address these challenges, we introduce LOKI, a novel approach that uses dynamic layer selection based on the Hilbert-Schmidt Independence Criterion and projects gradient updates onto the null-space of the model weights, bypassing the requirement for previous knowledge access. We show that LOKI achieves superior performance to existing approaches across a wide variety of experiments, achieving up to a 14\% improvement in average accuracy. We release our code at \href{https://github.com/neu-spiral/LOKI}{https://github.com/neu-spiral/LOKI}.
\end{abstract}

\section{Introduction}
Knowledge is a fundamental aspect of human intelligence, and the capabilities of any intelligent agent are largely dependent on its knowledge of the real world. Recently, Large Language Models (LLMs) have enjoyed tremendous success in a wide variety of real-life applications. However, upon deployment, LLMs often make mistakes~\cite{balachandran2022correcting}, suffer from hallucinations~\cite{ji2023survey}, and bias~\cite{ferrara2023should}. Furthermore, in an ever-evolving environment, the LLM's knowledge needs to be constantly updated with new facts. However, LLM knowledge is typically limited to what was present during pre-training, and retraining LLMs is incredibly costly and unsustainable for continual updates. To this end, \textit{lifelong knowledge editing} seeks to update LLM knowledge in a cheap and sustainable manner over time. 

A lifelong knowledge editing method should have three desirable properties~\cite{sinitsin2020editable, meng2022locating,yao2023editing, zhang2024comprehensive}: (i) \textit{reliability}, i.e., the edited LLM should reliably learn and recall the new knowledge presented to it, (ii) \textit{generalizability}, i.e., the LLM should be able to generalize and potentially reason about new knowledge rather than memorizing it, and (iii) \textit{locality}, i.e., learning new knowledge should not interfere with model capabilities and previous knowledge. Furthermore, the \textit{lifelong} nature of knowledge editing requires that the knowledge updates be made efficiently and sequentially. The process of editing new knowledge into an LLM roughly follows a three-phase structure (\cref{fig:main}): \textit{Layer Selection}, where a subset of model layers are chosen to be modified and the rest are frozen, \textit{Knowledge Consolidation}, where measures are taken to preserve the past knowledge already known by the model, and \textit{Knowledge Insertion}, where the new knowledge is to be learned by the model while making sure past knowledge is not forgotten.

Previous works have made attempts at lifelong knowledge editing in a variety of ways~\cite{zhang2024comprehensive}: by utilizing external knowledge~\cite{zheng2023can, zhong2023mquake}, auxiliary memory~\cite{hartvigsen2023aging, wang2024wise}, or editing intrinsic knowledge~\cite{meng2022mass, fang2024alphaedit}. In order to balance the incorporation of new knowledge and the preservation of past knowledge, compromises are made with regard to memory, computation, and added parameters, with some requiring expensive preprocessing steps for each LLM~\cite{meng2022mass, fang2024alphaedit}. Ideally, lifelong knowledge editing should be performed without significant computational overhead, access to external knowledge or memory, and without adding additional parameters.

To this end, we propose \textbf{L}ayer-Adaptive \textbf{O}rthogonal \textbf{K}nowledge \textbf{I}nsertion (LOKI), a novel lifelong knowledge editing method. Firstly, we propose a dynamic layer selection algorithm based on an information bottleneck criterion, utilizing the Hilbert-Schmidt Independence Criterion (HSIC)~\cite{gretton2005hsic} to identify the most important layers for a given fact to edit. Second, to consolidate previous knowledge, we propose to compute the null-space of the LLM weights themselves, without the need for access to examples of previous knowledge or extensive preprocessing. Finally, to insert the new knowledge directly into model weights of the previously selected layers, we perform projected gradient descent while using the previously computed null-space projections. Our contributions are as follows:
\begin{itemize}
    \item We propose a novel layer selection algorithm that uses HSIC Information Bottleneck to identify the most important layers to edit for a given fact. To the best of our knowledge, we are the first to explore per-sample layer selection.
    \item We propose using the null-space projection of the LLM weights to preserve past knowledge. We provide intuition and theoretical insight into this choice.
    \item We empirically verify the effectiveness of LOKI on various datasets, observing up to a $14\%$ increase in average performance. We conduct exploratory experiments and ablation studies into the components of our method.
\end{itemize}

\begin{figure*}
\centering
\label{fig1}
    \includegraphics[width=\textwidth]{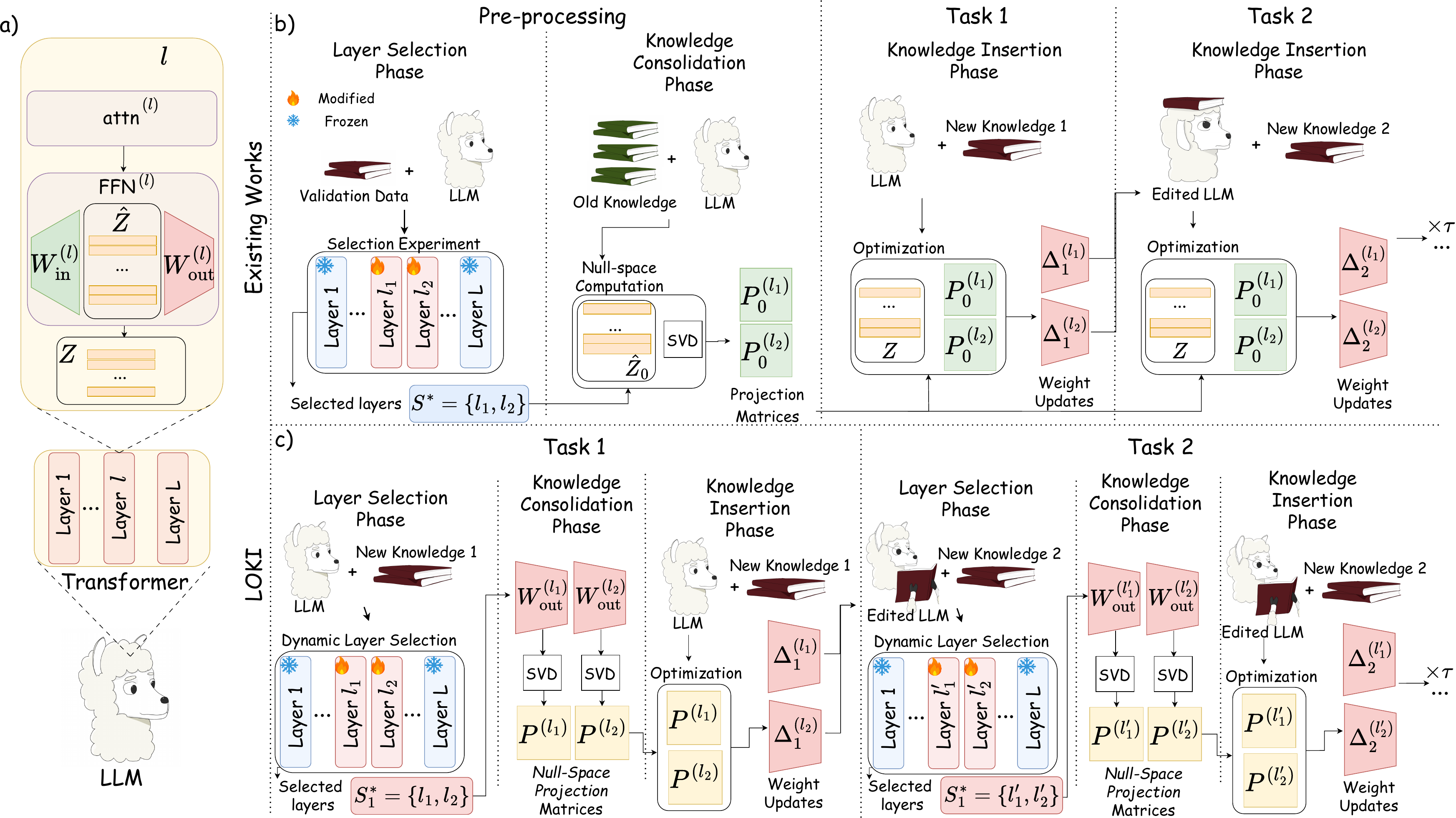}
\caption{ (a) Overview of an LLM, (b) the lifelong knowledge editing pipeline for existing works, and (c) the same pipeline for our method LOKI. Existing works typically decide a fixed set of layers during pre-processing and make all edits on those layers. Unlike existing works, LOKI selects important layers on a per-sample basis  and does not need access to previous knowledge samples or extensive pre-processing.}
\label{fig:main}
\end{figure*}

\section{Related Work}

\paragraph{Knowledge Editing.}
Existing works on knowledge editing of LLMs fall into three categories~\cite{zhang2024comprehensive}: resorting to external knowledge, merging knowledge into the model, and editing intrinsic knowledge. Methods such as IKE~\cite{zheng2023can} and MeLLo~\cite{zhong2023mquake} store facts in memory and retrieve them at inference time. WISE~\cite{wang2024wise} trains an auxiliary feed-forward network with activation routing, while GRACE~\cite{hartvigsen2023aging} maintains a discrete codebook as an adapter. Direct editing methods, most relevant to this work, modify model weights to incorporate new knowledge while preserving old. MEMIT~\cite{meng2022mass} updates parameters in batches, using statistics collected from Wikipedia samples for regularization. AlphaEdit~\cite{fang2024alphaedit} extends MEMIT to lifelong editing by projecting updates onto the null-space of old knowledge features. Our work differs from existing approaches by dynamically choosing which layer to modify on a per-sample basis, while existing works typically modify a fixed set of layers. Additionally, unlike previous works, our method does not require accessing or collecting the statistics of past knowledge during pre-processing, nor the addition of new parameters or codebooks.

\paragraph{Continual Learning.}
Machine Learning models typically struggle to learn from sequential training without having access to previous data, i.e., Continual Learning or Lifelong Learning. The central challenge of Continual Learning is a phenomenon called catastrophic forgetting~\cite{mccloskey1989catastrophic}, where the model's performance degrades significantly on previously learned data samples when exposed to new training data. Lifelong knowledge editing~\cite{wang2024wise, fang2024alphaedit} seeks to update model knowledge sequentially across multiple edits. In this scenario, the concern of catastrophic forgetting is twofold: the model may forget the knowledge that was learned during training, and it may also forget the new knowledge introduced by previous edits, when newer edits are performed. %\emph{Fine-tuning} is related but targets a different problem: they typically adapt a pretrained LLM to a new task, or domain, whereas knowledge editing targets discrete factual corrections without broadly repurposing the model~\cite{zhang2024comprehensive}.
%We tackle the issue of catastrophic forgetting through parameter isolation by way of dynamic layer selection and regularization through information bottleneck and null-space projection.

%\paragraph{Null-space Projection}
Null-space projection has emerged as a principled approach to mitigate catastrophic forgetting in continual learning by constraining gradient updates to directions that minimally interfere with previously learned knowledge~\cite{farajtabar2020ogd}. 
% The core insight is that neural networks are typically overparameterized, meaning parameter updates can be projected onto subspaces orthogonal to those important for prior tasks without significantly affecting past performance.
Orthogonal Gradient Descent (OGD)~\cite{farajtabar2020ogd} projects gradients for new tasks onto the orthogonal complement of gradients from previous tasks, ensuring minimal change to prior predictions. Gradient Projection Memory (GPM)~\cite{saha2021gpm} extends this idea by computing orthonormal bases of gradient subspaces via Singular Value Decomposition (SVD) on layer activations and storing them in memory for subsequent projection. 
% More recent work has explored scaled gradient projection~\cite{saha2023sgp}, which balances strict orthogonality with plasticity for new tasks by allowing scaled updates along important gradient subspaces. 
% In knowledge editing, AlphaEdit~\cite{fang2024alphaedit} projects weight updates onto the null-space of previous knowledge embeddings, requiring pre-computation of feature statistics from a large corpus.
Our approach differs fundamentally: we project gradients onto the null-space of the \emph{weight matrix} itself, eliminating the need for stored feature statistics or access to previous data while still preserving knowledge learned during training.

\paragraph{Hilbert-Schmidt Independence Criterion.}
%The Hilbert-Schmidt Independence Criterion (HSIC)~\cite{gretton2005hsic} is a kernel-based measure of statistical dependence between random variables. Given two random variables $X$ and $Y$ mapped to reproducing kernel Hilbert spaces (RKHSs) via kernels $k$ and $l$, HSIC is defined as the squared Hilbert-Schmidt norm of the cross-covariance operator between their feature representations; we provide the formal mathematical definition in \cref{app:hsic}. A key property is that with universal kernels, HSIC equals zero if and only if the variables are statistically independent, making it a robust non-parametric independence test.
The Hilbert-Schmidt Independence Criterion (HSIC)~\cite{gretton2005hsic} is a kernel-based measure of statistical dependence, defined as the squared Hilbert-Schmidt norm of the cross-covariance operator in reproducing kernel Hilbert spaces; we provide a formal definition in App.~\ref{app:hsic}. In deep learning, HSIC has been applied to train networks without backpropagation via the HSIC Bottleneck~\cite{ma2019hsicbottleneck}, which optimizes layers by maximizing dependence with outputs while minimizing dependence with inputs, an objective inspired by the Information Bottleneck principle~\cite{tishby2015deepib}. This framework enables efficient compression of learned representations by retaining task-relevant information while discarding redundant input details. We leverage this principle for dynamic layer selection: layers exhibiting high dependence with the output and low dependence with the input are identified as encoding task-critical knowledge and are prioritized for editing. Additionally, we incorporate HSIC as a regularization term during optimization to encourage information compression and reduce interference with previously learned knowledge.

%efine HSIC mathematically maybe in appendix
\section{Background}
\paragraph{Notation.} Let $f_{\theta}:\mathcal{X} \rightarrow\mathcal{Y}$ be an LLM parameterized by $\theta \in \mathbb{R}^{d}$, which maps an input $X \in \mathbb{R}^{ d_1\times k}$, consisting of $k\in \mathbb{N}$ tokens of dimension $d_1$, to a probability distribution $y \in \mathbb{R}^{|\mathcal{V}|}$ over the vocabulary of the model with size $|\mathcal{V}|$, denoting the probability of the next token in the sequence. The LLM consists of $L$ layers, each containing a self-attention module, and a feed-forward network (FFN). As shown in ~\cref{fig:main}(a), the FFN in layer $l$ of the LLM performs the operation:
\begin{align}
    z^{(l)}= \mathtt{FFN}^{(l)}(a^{(l)}) = W_{\mathtt{out}}^{(l)}\sigma(W_{\mathtt{in}}^{(l)}a^{(l)}) ,
\end{align}
where $a^{(l)} \in \mathbb{R}^{d_1}$ is a token in the output of the attention module, $z^{(l)} \in \mathbb{R}^{d_1}$ is the output,   $\sigma$ is an activation function,  and $W_{\mathtt{in}}^{(l)} \in \mathbb{R}^{d_2 \times d_1}$, and $W_{\mathtt{out}}^{(l)} \in \mathbb{R}^{d_1\times d_2}$ are up and down projection matrices (the parameters of the FFN). For brevity, we will denote the intermediate activation as $\hat{z}^{(l)}:=\sigma(W_{\mathtt{in}}^{(l)}a^{(l)})$, meaning $\mathtt{FFN}^{(l)}(a^{(l)}) = W_{\mathtt{out}}^{(l)}\hat{z}^{(l)}$. %Throughout the paper, we will use lower-case letters such as $z$ to denote vectors and scalars, and capital letters such as $Z$ to denote a batch of vectors or matrices.
% We use $\mathtt{Null}(Z)$ to denote the null-space of the matrix $Z$. The operation $\mathtt{FFN}^{(l)}:= z$ denotes replacing the output of the  $\mathtt{FFN}^{(l)}$ module with a given vector $z$. $\mathcal{L}(O|f_\theta(X))$ is shorthand for the negative log-likelihood loss function, depending on the output of the LLM given the subject and relation of the input $X =(S, R, O)$ and its corresponding desired object as output $O$.

\paragraph{Problem Setting.} During lifelong knowledge editing, a pre-trained LLM is to be adapted to sequentially arriving new knowledge. The goal of knowledge editing is to incorporate new knowledge while maintaining its predictive performance over past knowledge.
Formally, consider a pre-trained LLM $f_\theta$. A stream of tasks arrives sequentially, each comprising a single edit from the new knowledge dataset $\mathcal{D}_{\mathtt{edit}}=\{X_1, X_2, ..., X_\tau\}$: i.e., at task $t\leq \tau$, data point $X_t$ arrives, consisting of tokens partitioned into a subject, relation, and object, i.e., $X_t:=(S_t, R_t, O_t)$.  Let $\mathcal{L}$ be the negative log-likelihood loss function, and $\mathcal{D}_{\mathtt{past}} =\{X_{0,1}, X_{0,2},..., X_{0,n}\}$ a reference dataset of past knowledge. The goal of knowledge editing is to adjust weights $\theta$ after each task arrival so that (a) loss $\mathcal{L}(O_t|f_\theta(X_t))$ is minimized, while simultaneously (b) $\sum_{j=1}^n\mathcal{L}(O_{0,j}|f_\theta(X_{0,j}))+ \sum_{i=1}^{t-1} \mathcal{L}(O_i|f_\theta(X_{i}))$, i.e., the loss values for previous edits and past knowledge, does not increase significantly.

\paragraph{Knowledge Editing Pipeline.} Typically, knowledge editing algorithms follow three phases as depicted in \cref{fig:main}(b): (i) layer selection, where the layers to be edited are selected, (ii) knowledge consolidation, where measures are taken to prevent the forgetting of past knowledge, and (iii) knowledge insertion, where the new knowledge is inserted into the model. As a concrete example, we summarize  prior art closest to our work  w.r.t.~each of these three phases below.

\paragraph{MEMIT and AlphaEdit.}
As MEMIT~\cite{meng2022mass} and AlphaEdit~\cite{fang2024alphaedit} are particularly relevant to our method, we describe their methodology with greater detail in this section.

{\noindent\emph{(i) Layer Selection}} Meng et al.~\citep{meng2022locating},  preceding MEMIT, choose the editing layer during pre-processing: model inputs are corrupted by Gaussian noise, and subsequently, each token of each layer is denoised, and the effect is measured on the output. They conclude that knowledge in LLMs must be stored in the early layer FFNs, and hence fix the early layers as the same edited layers for all tasks. More specifically, only the $W_{\mathtt{out}}^{(l)}$ matrix of each  of these layers is modified. MEMIT and AlphaEdit follow the same assumptions and edit exactly the same layers and restricted parameters.

% We hypothesize that editing only a fixed set of layers for all samples can reduce flexibility, leading to increased interference with past knowledge. To this end, we propose a per-sample dynamic layer selection algorithm in \cref{sec:layersel}. To the best of our knowledge, we are the first to attempt per-sample layer selection for knowledge editing.
{\noindent\emph{(ii) Knowledge Consolidation.}} After choosing layers, both methods consolidate past knowledge via additional pre-processing (see Fig.~\ref{fig:main}(b)). Concretely, MEMIT summarizes past knowledge by computing the uncentered covariance matrix $C_0^{(l)} \in \mathbb{R}^{d_2\times d_2}$ of the past knowledge embeddings,
\vspace{-2mm}
\begin{align}
\label{eq:memitc}
    C_0^{(l)} = \hat{Z}_0\hat{Z}_0^\top,
\end{align}
where $\hat{Z}_0 \in \mathbb{R}^{d_2\times n}$ are the hidden $\mathtt{FFN}^{(l)}$ embeddings corresponding to $\mathcal{D}_\mathtt{past}$. AlphaEdit takes this a step further by computing the null-space of $\hat{Z}_0$ instead, resulting in a null-space projection matrix $P_0^{(l)}$, as depicted in \cref{fig:main}(b):
\begin{align}
\label{eq:alphanull}
%\begin{split}
    U_0^{(l)}\Sigma_0^{(l)} V_0^{(l)\top} &= \mathtt{SVD}(\hat{Z}_0\hat{Z}_0^\top) ,& &\text{and} &
    P_0^{(l)} &= \bar{V}_0^{(l)}\bar{V}_0^{(l)\top},
%    \end{split}
\end{align}
where $\bar{V}_0^{(l)}$ are the columns of ${V}_0^{(l)}$ corresponding to zero singular values.\footnote{Note that in practice, $\hat{Z}_0\hat{Z}_0^\top$ may be full rank and have an empty null-space, hence they use singular values below a certain threshold.} 
% Existing works such as~\cite{hartvigsen2023aging, wang2024wise} propose to add additional parameters or codebooks as a form of memory. While effective, such methods introduce even more parameters/memory to obscenely large language models. Similar to~\cite{meng2022locating, fang2024alphaedit}, we propose to edit the model parameters directly. Below, we will present briefly the formulation of MEMIT~\cite{meng2022mass} and AlphaEdit~\cite{fang2024alphaedit} as they are directly related to our method.

{\noindent\emph{(iii) Knowledge Insertion.}} To insert the new knowledge, MEMIT optimizes an intermediate vector inserted into the model activations. Concretely, they first obtain a desirable vector $z^{(l)*} \in \mathbb{R}^{d_1}$ \textit{for each sample} $X_t$, by inserting an optimizable vector at the output of a selected layer $l$'s FFN module:
\begin{align}
\label{eq:memitvec}
    z^{(l)*} = \arg\min_z \mathcal{L}(O_t |f_\theta(X_t))_{\mathtt{FFN}^{(l)} := z}, 
\end{align}
where $\mathtt{FFN}^{(l)} := z$ denotes replacing the output of $\mathtt{FFN}^{(l)}$ with $z$.
The resulting $z^{(l)*}$ vectors are aggregated in $Z^{(l)*} \in \mathbb{R}^{d_1 \times \tau}$ and then used to edit the model weights directly, while using aggregate statistics of past knowledge as a regularizer:
\begin{align}
\label{eq:memit}
    W_{\mathtt{out}}^{(l)*}  = \arg\min_W \|W\hat{Z}^{(l)} - Z^{(l)*}\|_2^2 + \|W\hat{Z}^{(l)}_0 - Z^{(l)}_0\|_2^2,
\end{align}
where $\hat{Z}^{(l)} \in \mathbb{R}^{d_2\times \tau}, \hat{Z}^{(l)}_0\in \mathbb{R}^{d_2\times n}$ are the intermediate $\mathtt{FFN}^{(l)}$ activations for the new and past knowledge respectively, and $Z_0^{(l)} \in \mathbb{R}^{d_1 \times n}$ is the output of the out projection, i.e., $Z_0^{(l)} = W_{\mathtt{out}}^{(l)}\hat{Z}^{(l)}_0$. There is an analytical solution to \cref{eq:memit}, using the $C_0^{(l)}$ matrix from \cref{eq:memitc}.

AlphaEdit follows the same procedure, but adjusts \cref{eq:memit} to include the null-space projection from \cref{eq:alphanull}, as well as the features of previous edits: \vspace{-1mm}
\begin{align}
    \label{eq:alpha}
    W_{\mathtt{out}}^{(l)*} = W_{\mathtt{out}}^{(l)} + \mathop{\arg\min}_\Delta \big\{ \|(W_{\mathtt{out}}^{(l)} + \Delta P^{(l)}_0)\hat{Z}^{(l)} - Z^{(l)*}\|_2^2 +  \|\Delta P^{(l)}_0\|_2^2 + \|\Delta P_0^{(l)}\hat{Z}^{(l)}_p\|_2^2\big\},
\end{align} 

\vspace{-2mm}where $\hat{Z}^{(l)}_p \in \mathbb{R}^{d_2\times t-1}$ are the $\mathtt{FFN}^{(l)}$ hidden embeddings corresponding to edits of previous tasks $\{1,\dots, t-1\}$. 
% Similar to MEMIT, AlphaEdit suffers from the need to compute $\hat{Z}_0$ as well. In the following section, we attempt to address this problem by substituting the null-space projection matrix $P$.
It is worth noting that the analysis by Meng et al.~\cite{meng2022locating} with regard to layer selection is limited to a single architecture and dataset. Furthermore, as mentioned in the appendix of~\cite{meng2022locating}, the denoising process is performed on ten consecutive layers at once, yet the effects are attributed to only the middle layer. Finally,  Wang et al.~\cite{wang2024wise} show that,  for some architectures, editing the mid-to-late layers is optimal in terms of empirical performance. All of this leads us to hypothesize that static layer selection is limiting the capabilities of knowledge editing algorithms. We propose a dynamic per-sample layer selection algorithm in \cref{sec:layersel}. Furthermore, both MEMIT and AlphaEdit require the explicit computation of $\hat{Z}_0$, typically with $n=100000$. Even though this needs to be done only once, it requires access to past knowledge samples, and it takes significant time and computation: roughly 56 hours of wallclock time to compute these statistics for the Llama3-8B-Instruct model on a single NVIDIA A100 GPU. This process also needs to be repeated for every new LLM. We address these challenges by proposing a substitute for the projection matrix obtained in \cref{eq:alphanull}, without requiring accessing or pre-processing previous knowledge (see \cref{sec:knowcon}). 

\section{LOKI: Layer-adaptive Orthogonal Knowledge Insertion}

 We present LOKI (\cref{fig:main}(c)), a novel knowledge editing method that seeks to overcome the shortcomings of existing works. We propose a per-sample layer selection algorithm allowing for more flexibility and less interference with past knowledge. Furthermore, we  bypass the requirement for access to previous knowledge, extensive pre-processing, additional parameters or memory. We describe our method in detail below.
 % First, for layer selection, we propose a novel dynamic layer selection algorithm based on an HSIC information bottleneck criterion. Second, for knowledge consolidation, we propose to compute the null-space of the model weights for the chosen layers, and finally, we describe the editing procedure, where we optimize the loss function over the selected layer parameters using projected gradient descent, while adding our proposed HSIC bottleneck term to the loss as additional regularization. 
% Finally, we provide theoretical insight into our choices regarding the null-space projection and provide a pseudo-code of our algorithm.
% only highlight the key differences in phases, and don't introduce notation
% First, we describe our dynamic layer selection algorithm, which selects a set of layers $S$ for each given input $X_t$, using an HSIC information bottleneck ($\mathtt{HIB})$ criterion. Second, we describe our approach to making each edit to the model. We propose to optimize the loss function $\mathcal{L}(O_t|f_\theta(X))$ over the FFN weights $W_{\mathtt{out}}^{(l)},\ \forall l\in S$, while projecting the weight updates $\Delta^{(l)}$ to the null-space of the corresponding weight matrix $W_{\mathtt{out}}^{(l)}$. We also add our layer selection 
\paragraph{Layer Selection Phase.}
\label{sec:layersel}
% Previous works have typically opted to edit only a small number of the LLM layers, most notably~\citet{meng2022mass, fang2024alphaedit} choose the first few layers of the LLM.  However, their analysis is limited to a single architecture and dataset. Furthermore, as mentioned in their appendix, the \textit{interventions} are performed on ten layers at once, but the effects of the interventions are attributed to only the layer in the center of all ten. Finally, other works such as~\cite{wang2024wise} opt to edit later layers in the model with great success, albeit by adding additional parameters. This leads us to believe that the knowledge within the model is not necessarily stored within the early layers for every sample, and calls for a per-sample selection strategy. 
We propose to use an HSIC-based information bottleneck criterion to dynamically identify which layers are more important for a given sample. To the best of our knowledge, this is the first attempt at dynamic layer selection for knowledge editing.
Let $X_t$ be the input embeddings of an edit sample at time $t$, and $\hat{Z}^{(l)}$ be the hidden FFN representations of the $l$-th layer. We define the following information bottleneck criterion $\mathtt{HIB}$ for a given layer $l$ and input $X_t$:
\begin{align}
    \label{eq:hib}
    \mathtt{HIB}(X_t, l) =  \lambda_y\mathtt{HSIC}(\hat{Z}^{(l)}, Z^{(L)}) - \lambda_x \mathtt{HSIC}(X_t, \hat{Z}^{(l)}),
\end{align}
where $\mathtt{HSIC}$  is the Hilbert-Schmidt Independence Criterion, $\lambda_x, \lambda_y$ are scalar hyperparameters, and $Z^{(L)}$ are the FFN output embeddings at the final layer, before the prediction head. $\mathtt{HIB}$ measures the flow of information between the input and output embeddings; where if a layer has higher mutual information with the output (i.e., higher $\mathtt{HSIC}$) and lower mutual information with the input, intuitively, that means it is holding more important information with regards to that sample, and discarding the noisy information.
To identify a subset $S^* \subset \{1,...,L\}$ of model layers with a fixed size $m$, we select the top $m$ layers in descending order of $\mathtt{HIB}(X_t, l)$. This is equivalent to optimizing the sum of $\mathtt{HIB}(X_t, l)$ over subsets of size $m$, i.e.,
%\begin{align}
 \(   S^* = \arg\max_{S\subset [L], |S|=m}\ \sum_{l \in S}  \mathtt{HIB}(X_t, l) .\)
%\end{align}
To compute $S^*$, we evaluate the $\mathtt{HIB}$ values for all layers using a single forward pass, and then choose the $m$ layers with the highest values, to be edited for that specific input sample $X_t$.

\paragraph{Knowledge Consolidation Phase.}
\label{sec:knowcon}
To consolidate previous knowledge, we propose to compute the null-space projection of the selected layers. The intuition behind this choice is that the information about past model knowledge should be ingrained into the weight matrix itself during training. We formalize this intuition in \cref{thm:null} further below, which provides theoretical insight into the relationship of the weight null-space and the null-space of the embeddings.
To this end, after selecting $S^*$ during the layer selection phase, we can compute the null-space projection of $W_{\mathtt{out}}^{(l)},$ where $ l\in S^*,$ using Singular Value Decomposition (SVD), as follows:%\vspace{-2mm}
\begin{align}
\label{eq:lokinull}
%\begin{split}
U^{(l)}\Sigma^{(l)} V^{(l)\top}& = \mathtt{SVD}(W_{\mathtt{out}}^{(l)}) & &\text{and} &
    P^{(l)} &= \bar{V}^{(l)}\bar{V}^{(l)\top}, & &\text{for all}~l\in S^*,
 %   \end{split}
\end{align}

\vspace{-2mm}where $P^{(l)} \in \mathbb{R}^{d_2\times d_2}$ is the projection matrix and $\bar{V}^{(l)} \in \mathbb{R}^{d_2 \times (d_2-d_1)}$ is made up of the columns of ${V}^{(l)}$ which correspond to zero singular values\footnote{Note that, unlike $\hat{Z}^{(l)}_0$, $W^{(l)}_{\mathtt{out}}$ is a low-rank matrix, since $d_2 > d_1$, and hence is guaranteed to have a non-empty null-space.}. It is important to note that in contrast to \cref{eq:alphanull}, solving \cref{eq:lokinull} does not require computing $\hat{Z}_0^{(l)}$, and can always be computed by only accessing $W_{\mathtt{out}}^{(l)}$, saving on time, computation, and memory.

The following theorem sheds light on our intuition behind choosing the weight null-space, and the relationship it has with the feature null-space as used in \cref{eq:alphanull}: 
\begin{restatable}{theorem}{WeightNullspace}
\label{thm:null}
    Let $W_{\mathtt{out}}^{(l)} \in \mathbb{R}^{d_1\times d_2}$ be the down projection matrix for the $l$-th layer, learned using gradient descent on samples $X_{0,1},...,X_{0,n}$. For each sample $X_{0,i} \in \mathbb{R}^{d_1\times k}$, let $\tilde{z}_{0,i}^{(l)} \in \mathbb{R}^{d_2 \times k}$ be the corresponding token hidden embeddings at the input of $W_{\mathtt{out}}^{(l)}$ \underline{at the time of training}. Let $\tilde{Z}^{(l)}_0 \in \mathbb{R}^{d_2\times n}$ be the concatenation of all $\tilde{z}_{0,i}^{(l)}$. Then %\vspace{-2mm}
    %\begin{align}
     \(   \mathtt{Null}(\tilde{Z}^{(l)\top}_0) \subset \mathtt{Null}(W_{\mathtt{out}}^{(l)}).\)
    %\end{align}
\end{restatable}

\vspace{-2mm}A proof of the theorem is presented in App.~\ref{app:proof}, and follows the fact that the gradient updates to $W_{\mathtt{out}}^{(l)}$ are an outer product of the features during training. 

Note that the above theorem is regarding the model features during training, not after training has completed. Therefore, we must make a simplifying assumption: Let $\hat{z}_{0,i}^{(l)}$ be the token hidden embeddings corresponding to $X_{0,i}$ at the input of $W_{\mathtt{out}}^{(l)}$ \underline{after training} and $\hat{Z}_0^{(l)} \in \mathbb{R}^{d_2\times n}$ be the concatenation of said embeddings. Assuming $\hat{z}_{0,i}^{(l)} \approx \tilde{z}_{0,i}^{(l)},\ \forall i$, then we have \vspace{-2mm}
\begin{align}
    \mathtt{Null}(\hat{Z}^{(l)\top}_0) \subset \mathtt{Null}(W_{\mathtt{out}}^{(l)})
\end{align}

\vspace{-3mm}The above states that if we assume the model embeddings, and consequently their span, have not changed significantly from training time, the resulting null-spaces are therefore overlapping, and the null-space of the trained model features is also a subset of the weight null-space. While this is a strong assumption, the similarity of model features during and after training is supported by the observations in Neural Tangent Kernel (NTK) literature~\cite{chizat2019lazy}. They observe that during the training of neural networks, the model training enters a lazy regime, and model parameters hardly vary. Furthermore, we empirically validate the effectiveness of projecting the gradients onto the weight null-space in our experiments and ablation studies.

% \footnote{Since $P^{(l)}$ can always be computed from the model easily, it does not need to be stored along with the model}.

% Furthermore, the null-space of $W^{(l)}_{\mathtt{out}}$ can always be computed as described above, meaning it does not need to be stored on disk, \underline{adding no parameters to the model.} refer to the theorem and say high level intuition

% compare this to previous works also

\begin{algorithm}[!t]
\scriptsize
\caption{LOKI Knowledge Editing}\label{alg:loki}
\begin{algorithmic}
\REQUIRE Editing dataset $\mathcal{D}_{\mathtt{edit}}$ with $\tau$ tasks, LLM $f_\theta$, Num edit layers $m$, Num steps $k$, Adjusted Norm constraint $\eta'$, Learning Rate $\alpha$
\FOR{$t = 1...\tau$}
\STATE $X_t, O_t \gets \mathcal{D}_{\mathtt{edit}}[t]$ \hfill \COMMENT{Input arrives}
\STATE $\hat{Z}^{(1)}\ \dots\ \hat{Z}^{(L)}, Z^{(L)} \gets \{f_\theta(X_t)\}$ \hfill\COMMENT{Hidden Embeddings}\\
\STATE $\mathtt{scores} \gets \{\mathtt{HIB}(X_t, l)\quad \mathbf{for}\quad l=1\dots L\}$
% \FOR{$l=1\dots L$}
% \STATE $\mathtt{scores}.\mathtt{add}\{ \mathtt{HIB}(X_t, l)\}$\hfill\COMMENT{\cref{eq:hib}}
% \ENDFOR
\STATE $S^* \gets \mathtt{top\_m}(\mathtt{scores})$\hfill\COMMENT{Selected Layers}
\FOR{$l \in S^*$}
\STATE $P^{(l)} \gets \mathtt{SVD}(W_{\mathtt{out}}^{(l)})$ \hfill\COMMENT{\cref{eq:lokinull}}
\ENDFOR
\FOR{$i=1\dots k$}
\STATE $\mathcal{L}_{\mathtt{LOKI}} \gets \mathcal{L}(O_t|f_\theta(X_t)) - \sum_{l\in S^*}\mathtt{HIB}(X_t,l)$
\FOR{$l \in S^*$}
\STATE $g \gets (\nabla_{W_{\mathtt{out}}^{(l)}}\mathcal{L}_{\mathtt{LOKI}}) P^{(l)}$ \hfill \COMMENT{Project to Null-Space}
% \IF{$\|g\|_2 > \eta'$}
\vspace{1mm}
\STATE $g \gets \min(\frac{\eta'}{\|g\|_2}, 1)\cdot g$ \hfill \COMMENT{Norm Clip}
% \ENDIF
\STATE $W_{\mathtt{out}}^{(l)} \gets W_{\mathtt{out}}^{(l)} - \alpha \cdot g$
\ENDFOR
\ENDFOR
\ENDFOR
\end{algorithmic}
\end{algorithm}

\paragraph{Knowledge Insertion Phase.}
In this section, we describe the procedure of inserting knowledge into the model weights. Given an input, we propose to optimize the negative log-likelihood of the desired prediction given each input with the added regularization terms, using gradient descent. During optimization, we project the gradients for each layer to the null-space of that layer, using the previously obtained matrices in \cref{eq:lokinull}.
In addition to the null-space projection, we add the HSIC Information bottleneck introduced in \cref{eq:hib} to the optimization objective, described in the following section. Maximizing the $\mathtt{HIB}$ for each layer leads to efficient compression of the information learned by the model, reducing interference with past knowledge. Concretely, we add the  term 
%\begin{align}
 \(   \sum_{l \in S^*}-\mathtt{HIB}(X_t, l)\)
%\end{align}
to the loss function for task~$t$.
Furthermore, we restrict the norm of the weight updates for stability, that is, less interference with previous knowledge.
% To address the issues with previous works and preserving past knowledge, we propose to project each weight update to the null-space of its corresponding weight matrix $W_{\mathtt{out}}^{(l)}$, using the projection matrix $P^{(l)}$. The intuition behind this choice is that the information about past model knowledge should be ingrained into the weight matrix itself during training. In the sections below, we provide additional theoretical insights into this choice and empirically demonstrate its effectiveness using experiments and ablation studies. not here
% fix the S*
% Furthermore, we add the negative sum of the layer selection criterion, $-\sum_{l\in S}\mathtt{HIB}(X_t,l)$, to the optimization objective to further compress the information learned by the model, leading to less interference with past knowledge.
% Additionally, we propose to constrain the norm of the updates to prevent large updates to the parameters. The final optimization problem is as follows: 
Formally, we solve the following optimization problem:%\vspace{-2mm}
\begin{align}
\label{eq:finalopt}
\begin{split}
    \min_{\mathbf{\Delta}}\ \mathcal{L}(O_t|f_\theta(X_t)) - \sum_{l\in S^*}\mathtt{HIB}(X_t,l) \quad
    \text{s.t.}\ W_{\mathtt{out}}^{(l)}\Delta^{(l)T} = 0, \quad \| \Delta^{(l)}\|_F \leq \eta, \quad \forall \Delta^{(l)} \in \mathbf{\Delta}
    \end{split}
\end{align}
where $\mathbf{\Delta} = \{ \Delta^{(l)}|l \in S^*\}$ is the set of weight updates for the selected layers $S^*$. In contrast to AlphaEdit/MEMIT, this approach allows the direct optimization of the objective, without the need for an intermediate vector (\cref{eq:memitvec}). 
% \textbf{blue}{Compared to xxxx, we  EXPLAIN HOW THIS IS DIFFERENT/ BETTER THAN ALPHA EDIT.}

As the objective functions are non-convex, we opt to solve \cref{eq:finalopt} using projected gradient descent, where the gradients of the objective with respect to the weights are calculated using backpropagation, then projected using the projection matrices $P^{(l)}$, and finally normalized if needed to preserve the norm constraint. We present a detailed pseudocode of LOKI in \cref{alg:loki}, and refer the reader to App.~\ref{app:impdetails} for additional implementation details and runtime analysis. We implement our method using the EasyEdit \cite{wang2023easyedit} library, and the code can be found at \href{https://github.com/neu-spiral/LOKI}{https://github.com/neu-spiral/LOKI}.

\section{Experiments}
\label{sec:exp}
% make the sections clear (setup - results)
% Maybe have an explanation for methods/datasets in the supplement 

%In this section, we empirically evaluate the effectiveness of LOKI across various benchmarks and compare it to existing knowledge editing methods. Furthermore, we conduct exploratory experiments and ablation studies to provide insights into the components of our method.

% \subsection{Experimental Setup}
We provide a detailed description of our experimental setup in App.~\ref{app:expdetails}; we briefly summarize it here. 

\paragraph{Baselines.} We compare to SimIE~\cite{guo2025towards}, AlphaEdit~\cite{fang2024alphaedit}, ROME~\cite{meng2022locating}, and MEMIT~\cite{meng2022mass} as baselines that modify model weights directly without the addition of new parameters. Additionally, we compare to WISE~\cite{wang2024wise}, GRACE~\cite{hartvigsen2023aging}, {and DEFER~\cite{mitchell2022memory, hartvigsen2023aging}, as baselines which utilize external memory, codebooks, or parameters. Finally, we include Fine-Tuning (FT) {following \cite{zhang2024comprehensive}} and AdaLoRA~\cite{zhang2023adaptive} as a lower bound on performance.
% \paragraph{Architectures} Following previous works~\cite{wang2024wise}, we utilize three open-source LLM models: Llama3-8B-Instruct~\cite{grattafiori2024llama}, Mistral-7B~\cite{jiang2023mistral7b}, and GPT-J-6B~\cite{gpt-j}.

\paragraph{Architectures, Datasets \& Metrics.} Following previous works~\cite{wang2024wise}, we utilize three open-source models: Llama-3-8B-Instruct~\cite{grattafiori2024llama}, Mistral-7B~\cite{jiang2023mistral7b}, and GPT-J-6B~\cite{gpt-j}. Following~\cite{wang2024wise}, we evaluate LOKI on the \emph{question answering} dataset ZsRE~\cite{levy2017zero}, the hallucination correction dataset SelfCheckGPT~\cite{manakul2023selfcheckgpt}, and finally, to measure out of distribution edit effectiveness, the Temporal dataset~\cite{hewitt2024model}. We measure the effectiveness of editing algorithms using three main metrics: Reliability, Generalization, and Locality. The concrete score reported for each metric is benchmark-specific. Full definitions for ZsRE, SelfCheckGPT, and Temporal are given in App.~\ref{app:expdetails}.
% \paragraph{Metrics} Following existing literature~\cite{wang2024wise}, we measure the effectiveness of editing algorithms using three main metrics: Reliability, Generalization, and Locality. 

% \subsection{Experiment Results}

\paragraph{Main Editing Results.}
\begin{table*}[t!]
%\caption{Results of editing two LLMs on the ZsRE dataset, with $\tau$ being the number of sequential edits. Avg represents the average of all three metrics. Higher is better for all metrics. LOKI outperforms existing methods consistently, especially so with a larger number of edits.}
\caption{Results on the ZsRE dataset (Llama-3-8B-Instruct and Mistral-7B). Metrics include Reliability (Rel.~$\uparrow$), Generalization (Gen.~$\uparrow$), and Locality (Loc.~$\uparrow$) across $\tau$ sequential edits, with Avg. representing the mean of the three scores. \textbf{Bold} is the best result, and underline is the second-best. LOKI consistently achieves the highest average performance across all edit scales, demonstrating superior scalability and stability compared to existing methods.}% as $\tau$ increases up to 1,000.}
\centering
\label{tab:main}
\resizebox{\linewidth}{!}{
\begin{tabular}{c|ccc|c|ccc|c|ccc|c|ccc|c}
\toprule
Method    & \multicolumn{4}{c|}{$\tau=$ 1}              & \multicolumn{4}{c|}{$\tau=$ 10}               & \multicolumn{4}{c|}{$\tau=$ 100}              & \multicolumn{4}{c}{$\tau=$ 1000}             \\
\midrule
          & Rel. ($\uparrow$)  & Gen. ($\uparrow$)    & Loc. ($\uparrow$)    & Avg. ($\uparrow$)         & Rel. ($\uparrow$)    & Gen. ($\uparrow$)    & Loc. ($\uparrow$)    & Avg. ($\uparrow$)         & Rel. ($\uparrow$)    & Gen. ($\uparrow$)    & Loc. ($\uparrow$)    & Avg. ($\uparrow$)         & Rel. ($\uparrow$)    & Gen. ($\uparrow$)    & Loc. ($\uparrow$)    & Avg. ($\uparrow$)         \\
          \midrule
\multicolumn{17}{c}{Llama-3-8B-Instruct}                         \\
\midrule
FT        & \textbf{1} & \textbf{0.99}  & 0.95   & \underline{0.98}        & 0.87   & 0.83   & 0.80   & 0.83        & 0.84   & 0.81   & 0.42   & 0.69        & 0.83   & 0.80   & 0.11   & 0.58        \\
%LoRA      & \textbf{1} & \textbf{0.99}  & 0.49   & 0.83        & 0.76   & 0.72   & 0.32   & 0.60        & 0.47   & 0.45   & 0.30   & 0.41        & 0.36   & 0.35   & 0.20   & 0.31        \\
AdaLoRA   & \textbf{1} & 0.73   & 0.82   & 0.85        & 0.83   & 0.60   & 0.57   & 0.67        & 0.54   & 0.49   & 0.44   & 0.49        & 0.36   & 0.35   & 0.26   & 0.32        \\
ROME      & \underline{0.99} & 0.96   & 0.96   & 0.97        & 0.61   & 0.59   & 0.41   & 0.54        & 0.08   & 0.08   & 0.02   & 0.06        & 0.03   & 0.03   & 0.02   & 0.03        \\
MEMIT     & \underline{0.99} & 0.97   & 0.98   & 0.98        & 0.68   & 0.67   & 0.73   & 0.69        & 0.03   & 0.03   & 0.01   & 0.02        & 0.00      & 0.00      & 0.00      & 0.00           \\
DEFER     & 0.82 & 0.86   & 0.68   & 0.79        & 0.66   & 0.67   & 0.45   & 0.59        & 0.33   & 0.32   & \textbf{1}      & 0.55        & 0.31   & 0.31   & \textbf{1}      & 0.54        \\
GRACE     & \textbf{1}    & 0.46   & \textbf{1}      & 0.82        & \textbf{1}      & 0.42   & \textbf{1}     & 0.81        & \textbf{1}      & 0.39   & \textbf{1}     & 0.8         & \textbf{1}      & 0.37   & \textbf{1}     & \underline{0.79}        \\
SimIE & 0.75 & 0.71 & 0.79 & 0.75 & 0.73 & 0.69 & 0.79 & 0.73 & 0.72 & 0.67 & 0.78 & 0.72 & 0.69 & 0.65 & 0.75 & 0.69 \\
WISE      & \underline{0.99} & \textbf{0.99}   & 0.92      & 0.97        & 0.77   & 0.74   & \textbf{1}   & 0.84        & 0.61   & 0.6    & \textbf{1}      & 0.73        & 0.54   & 0.53   & \underline{0.99}      & 0.69        \\
AlphaEdit & \underline{0.99} & 0.88   &  \underline{0.99}   & 0.95        & 0.92   & \underline{0.84}   & \underline{0.98}   & \underline{0.92}        & \underline{0.90 }  &\underline{0.84}   & 0.94   & 0.90        & 0.76   & \textbf{0.88}   & 0.65   & 0.76        \\
\midrule
LOKI      & \textbf{1}    & \underline{0.98} & \underline{0.99}  & \textbf{0.99} & \underline{0.94} & \textbf{0.87} & \underline{0.98} & \textbf{0.93} & \textbf{0.93} & \textbf{0.86}& \underline{0.96} & \textbf{0.92}      & \underline{0.91} & \underline{0.87}  & \underline{0.88} & \textbf{0.89} \\
\midrule
\multicolumn{17}{c}{Mistral-7B}  \\
\midrule
FT        & \underline{0.99} & \textbf{0.98}   & 0.75   & 0.91        &  0.82   & 0.77    & 0.22        & 0.60   & 0.83   & \underline{0.77}    & 0.07   & 0.56        & 0.69   & 0.65   & 0.07   & 0.47         \\
%LoRA      & \textbf{1} & \textbf{1}   & 0.31   & 0.77        & 0.63   & 0.60   & 0.21   & 0.48        & 0.33   & 0.32   & 0.17   & 0.28        & 0.15   & 0.15   & 0.07   & 0.12         \\
AdaLoRA   & \textbf{1} & 0.91   & 0.84   & 0.92        & 0.91   & 0.78   & 0.62   & 0.77        & 0.56   & 0.51   & 0.43   & 0.50        & 0.50   & 0.48   & 0.40   & 0.46         \\
ROME      & 0.79 & 0.77   & 0.98   & 0.85        & 0.58   & 0.57   & 0.75   & 0.63        & 0.05   & 0.05   & 0.02   & 0.04        & 0.04   & 0.04   & 0.02   & 0.03        \\
MEMIT     & 0.81 & 0.79   & \underline{0.99}   & 0.86        & 0.46   & 0.45   & 0.61   & 0.51        & 0.00      & 0.00      & 0.01   & 0.00           & 0.04   & 0.04   & 0.02   & 0.03        \\
DEFER     & 0.59 & 0.68   & 0.79   & 0.69        & 0.44   & 0.43   & \textbf{1}      & 0.62        & 0.40    & 0.40    & \textbf{1}      & 0.60         & 0.40    & 0.39   & \textbf{1}      & 0.60         \\
GRACE     & \textbf{1}    & 0.36   & \textbf{1}      & 0.79        & \textbf{1}      & 0.15   & \textbf{1}      & 0.72        & \textbf{1}      & 0.15   & \textbf{1}      & 0.72        & \textbf{1}      & 0.02   & \textbf{1}      & 0.67        \\
%SimIE & 62.22 & 58.46 &93.69 &71.46 &60.55&57.89 &93.56 &70.67 &59.95 &57.58 &93.32 &70.28 &58.51 & 56.15 &91.54 & 68.73\\
SimIE & 0.62 & 0.59 & 0.94 & 0.72 & 0.61 & 0.58 & 0.94 & 0.71 & 0.60 & 0.58 & 0.93 & 0.70 & 0.59 & 0.56 & 0.92 & 0.69 \\
WISE      & 0.96 & 0.93   & \textbf{1}      & \underline{0.96}        & 0.88   & \underline{0.84}   & \underline{0.99}      & \underline{0.91}        & 0.82   & 0.76    & \underline{0.99}      & \underline{0.86}        & 0.65    & 0.63   & \underline{0.99}     & \underline{0.76}       \\
AlphaEdit & 0.86 & 0.78   & \underline{0.99}      & 0.88        & 0.82   & 0.76   & \underline{0.99}   & 0.86        & 0.80   & 0.75   & 0.96   & 0.84        & 0.80   & \underline{0.74}   & 0.75   & \underline{0.76}        \\
\midrule
\textbf{LOKI} (Ours)      & \textbf{1}    & \underline{0.96 }& 0.98 & \textbf{0.98}      & \underline{0.96} & \textbf{0.91} & 0.97 & \textbf{0.95}      & \underline{0.95}  & \textbf{0.88} & 0.93 & \textbf{0.92} & \underline{0.92}& \textbf{0.86} & 0.82 & \textbf{0.87}\\
\bottomrule
\end{tabular}
}
\end{table*}

We present results for ZsRE, SelfCheckGPT, and Temporal in \cref{tab:main}, \cref{tab:hallucination}, and \cref{tab:ood_generalization}, respectively. We report deviations across independent runs for the same settings in App.~\ref{app:extended}. Overall, LOKI outperforms existing methods in most scenarios. In \cref{tab:main}, LOKI achieves top average performance across all scenarios, with a particularly stark performance gap at $\tau=1000$ on both models, achieving up to $14$\% improvement in average performance over the runner-up. While some methods perform well in only one or two metrics for each experiment, LOKI is able to balance all three consistently. In \cref{tab:hallucination}, LOKI is able to consistently {achieve the best reliability score (up to $84$\% improvement over the runner-up for $\tau=600$) and remain close to the best locality scores (less than $15$\% degradation)}. It is worth noting that methods with higher locality scores suffer from a minimum $47$\% degradation in perplexity for multiple edits. Similarly, in \cref{tab:ood_generalization}, LOKI achieves the best or second-best Reliability and Generalization score, while remaining competitive in Locality (less than $20$\% degradation), attaining the second-best average score. 

\begin{table*}[!t]
\centering
\caption{Performance comparison on the SelfCheckGPT hallucination correction task. Metrics include Reliability (Rel. $\downarrow$), as target-text perplexity (PPL) and Locality (Loc. $\uparrow$) across $\tau$ sequential edits; we omit generalization because paraphrase prompts are unavailable for this task. \textbf{Bold} is the best result, and underline is the second-best. LOKI demonstrates superior Rel. and competitive Loc. throughout the editing process.}
\resizebox{\textwidth}{!}{
\begin{tabular}{c|cc|cc|cc|cc|cc|cc|cc|cc}
\toprule
Method & \multicolumn{8}{c|}{Llama-3-8B-Instruct} & \multicolumn{8}{c}{Mistral-7B} \\
\midrule
 & \multicolumn{2}{c|}{$\tau=$ 1} & \multicolumn{2}{c|}{$\tau=$ 10} & \multicolumn{2}{c|}{$\tau=$ 100} & \multicolumn{2}{c|}{$\tau=$ 600} & \multicolumn{2}{c|}{$\tau=$ 1} & \multicolumn{2}{c|}{$\tau=$ 10} & \multicolumn{2}{c|}{$\tau=$ 100} & \multicolumn{2}{c}{$\tau=$ 600} \\
 & Rel. ($\downarrow$) & Loc. ($\uparrow$) & Rel. ($\downarrow$) & Loc. ($\uparrow$) & Rel. ($\downarrow$) & Loc. ($\uparrow$) & Rel. ($\downarrow$) & Loc. ($\uparrow$) & Rel. ($\downarrow$) & Loc. ($\uparrow$) & Rel. ($\downarrow$) & Loc. ($\uparrow$) & Rel. ($\downarrow$) & Loc. ($\uparrow$) & Rel. ($\downarrow$) & Loc. ($\uparrow$) \\
\midrule
FT & \textbf{ 1.01} & 0.95 & \underline{1.12} & 0.84 & \underline{1.44} & 0.65 & \underline{2.02} & 0.45 & \underline{1.02} & 0.77 & \underline{1.41} & 0.44 & 1.94 & 0.22 & \underline{7.26} & 0.20 \\
ROME & 1.39 & 0.97 & 4.74e1 & 0.61 & 5.00e4 & 0.02 & 2.89e4 & 0.01 & 1.98 & \underline{0.99} & 2.58 & 0.91 & 6.77e3 & 0.03 & 7.87e3 & 0.01 \\
MEMIT & 1.32 & \underline{0.99} & 1.84e3 & 0.93 & 6.25e10 & 0.01 & 3.75e4 & 0.05 & 1.59 & \textbf{1} & 3.13e2 & 0.93 & 9.03e4 & 0.00 & 3.33e6 & 0.00 \\
DEFER & 2.34e1 & 0.80 & 6.86e1 & \textbf{1} & 7.12e1 & \textbf{1} & 7.12e1 & \textbf{1} & 1.76e1 & 0.92 & 3.22e1 & \underline{0.99} & 3.22e1 & \textbf{1} & 3.22e1 & \textbf{1} \\
GRACE & 1.40 & \textbf{1} & 7.10e1 & \textbf{1} & 7.17e1 & \textbf{1} & 7.54e1 & \textbf{1} & 1.92 & \textbf{1} & 3.21e1 & \textbf{1} & 3.22e1 & \textbf{1} & 3.25e1 & \textbf{1} \\
WISE & 1.44 & 0.97 & 3.47 & 0.97 & 3.54 & \underline{0.99} & 2.63e2 & \underline{0.99} & 1.08 & \textbf{1} & 1.43e1 & 0.98 & \underline{1.74} & 0.96 & 10.0 & \underline{0.92} \\
AlphaEdit & \underline{1.29} & \textbf{1} & 1.49 & \underline{0.99} & 5.04 & 0.96 & 9.04e3 & 0.05 & 1.71 & \textbf{1} & 1.88 & \underline{0.99} & 2.96 & \underline{0.98} & 6.43e1 & 0.87 \\
%MEMOIR & 1 & 1 & 1.01 & 1 & 1.07 & 1 & 1.25 & 1 & 1 & 1 & 1.02 & 1 & 1.09 & 1 & 1.22 & 1 \\
\midrule
%\textbf{LOKI} (Ours) & \textbf{1.01} & 0.98 & \textbf{1.01} & 0.97 & \textbf{1.05} & 0.89 & \textbf{1.16} & 0.78 & \textbf{1.01} & \underline{0.99} & \textbf{1.02} & 0.96 & \textbf{1.08} & 0.87 & \textbf{1.43} & 0.72 \\
\textbf{LOKI} (Ours) & \textbf{1.01} & 0.98 & \textbf{1.01} & 0.98 & \textbf{1.03} & 0.94 & \textbf{1.14} & 0.85 & \textbf{1.01} & \underline{0.99} & \textbf{1.01} & \underline{0.99} & \textbf{1.04} & 0.94 & \textbf{1.13} & 0.87 \\
\bottomrule
\end{tabular}
}
\label{tab:hallucination}
\end{table*}

\begin{table*}[!t]
\centering
\begin{minipage}{.49\textwidth}
  \centering
  \tiny
\resizebox{\textwidth}{!}{
\begin{tabular}{l|ccc|c|ccc|c}
\toprule
\multirow{2}{*}{Method} & \multicolumn{4}{c|}{$\tau=$ 10} & \multicolumn{4}{c}{$\tau=$ 75} \\
\cmidrule(lr){2-5} \cmidrule(lr){6-9}
& Rel. & Gen. & Loc. & Avg. & Rel. & Gen. & Loc. & Avg. \\
 & ($\uparrow$) & ($\uparrow$) & ($\uparrow$) & ($\uparrow$) & ($\uparrow$) & ($\uparrow$) & ($\uparrow$) & ($\uparrow$) \\
\midrule
w/o Editing & 0.56 & 0.21 & --- & --- & 0.56 & 0.21 & --- & --- \\
\midrule
FT & 0.96 & \underline{0.35} & 0.69 & 0.67 & 0.93 & 0.32 & 0.62 & 0.62 \\
ROME & 0.10 & 0.00 & 0.06 & 0.05 & 0.02 & 0.00 & 0.02 & 0.01 \\
MEMIT & 0.82 & 0.28 & \underline{0.99} & 0.70 & 0.43 & 0.14 & 0.49 & 0.35 \\
DEFER & 0.66 & 0.29 & 0.81 & 0.59 & 0.55 & 0.21 & 0.95 & 0.57 \\
GRACE & 0.59 & 0.22 & \textbf{1} & 0.60 & 0.56 & 0.22 & \textbf{1} & 0.59 \\
WISE & \underline{0.99} & \textbf{0.36} & 0.93 & \textbf{0.76} & \textbf{0.98} & \textbf{0.38} & \underline{0.99} & \textbf{0.78} \\
AlphaEdit & 0.90 & 0.27 & \underline{0.99} & 0.72 & \underline{0.87} & 0.28 & 0.97 & 0.71 \\
\midrule
\textbf{LOKI} (Ours) & \textbf{1} & \textbf{0.36} & 0.82 & \underline{0.73} & \textbf{0.98} & \underline{0.36} & 0.81 & \underline{0.72} \\
\bottomrule
\end{tabular}}
% }
  \captionof{table}{OOD performance on the Temporal dataset (GPT-J-6B). Metrics include Reliability (Rel.~$\uparrow$), OOD Generalization (Gen.~$\uparrow$), and Locality (Loc.~$\uparrow$) across $\tau$ sequential edits. \textbf{Bold} is the best result, and underline is the second-best. LOKI maintains high Rel., and Gen., and competitive Loc. across the editing sequence.}
\label{tab:ood_generalization}
\end{minipage}%
\hfill
\begin{minipage}{.49\textwidth}
  \centering
  \includegraphics[width=0.85\linewidth]{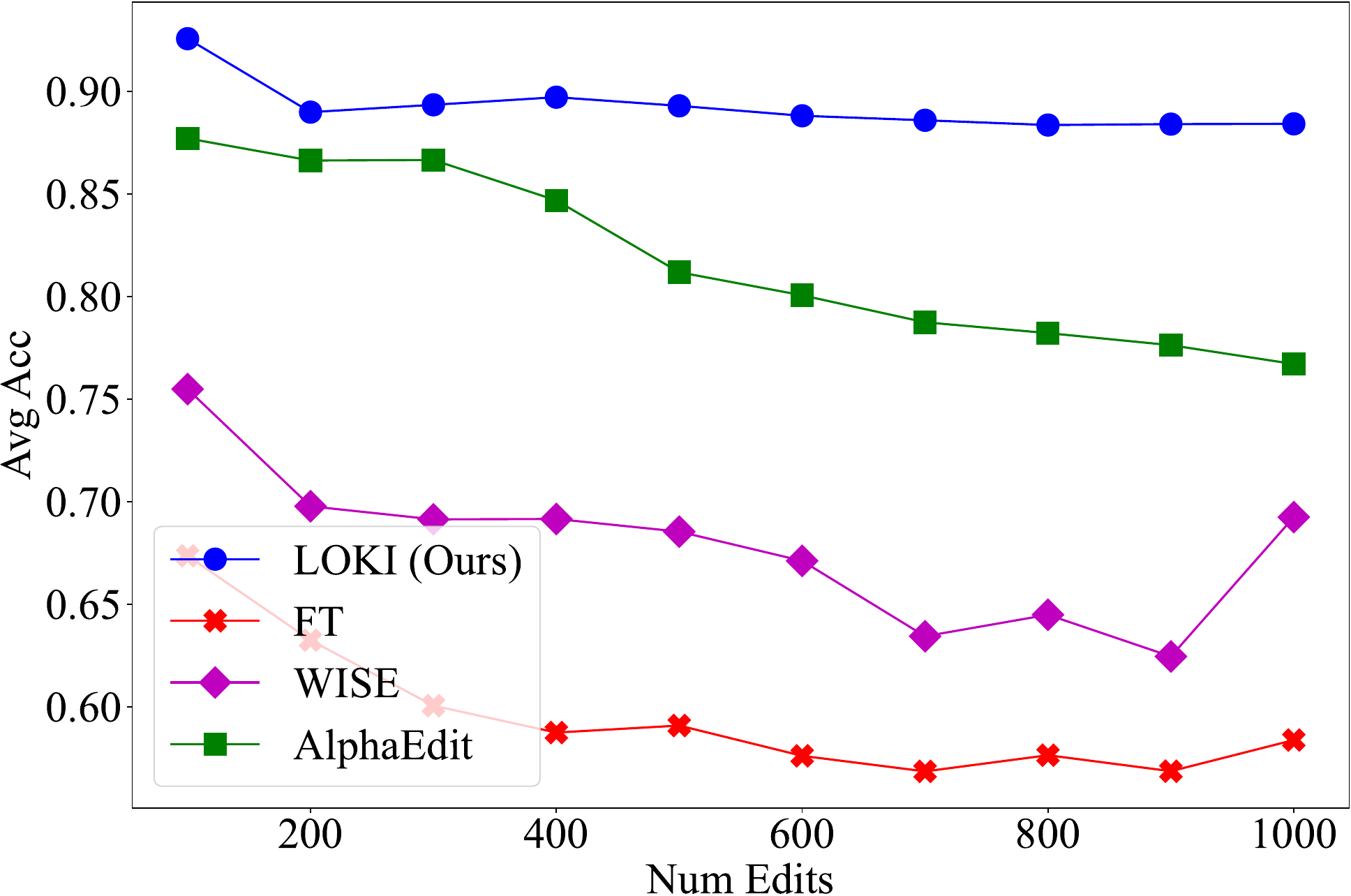}
  \captionof{figure}{Average accuracy of various methods every 100 edits for 1000 sequential edits on the  Llama-3-8B-Instruct and ZsRE dataset. LOKI suffers the least performance degradation over time across competitors.}
  \label{fig:overtime}
\end{minipage}
\end{table*}

\paragraph{Runtime Efficiency.} We report runtime in \Cref{fig:edit_time_comparison} (App.~\ref{app:impdetails}), which measures average edit time per sample on SelfCheckGPT over 600 edits. Even without preprocessing cost, LOKI (17.6s) is on average faster than WISE (19.9s) and AlphaEdit (20.2s), the most competitive methods across our tested benchmarks. Accounting for preprocessing makes AlphaEdit roughly $2\times$ slower than LOKI. % only on AlphaEdit's edited layers

 \paragraph{Evolution over Tasks and Layer Selection Frequency.} It is also of interest to measure performance across tasks in regular intervals. In \cref{fig:overtime}, we visualize the average performance of LOKI, WISE, AlphaEdit and FT every 100 tasks across 1000 tasks, on the ZsRE dataset and the Llama-3 model. LOKI maintains superior accuracy over time and suffers less deterioration compared to competitors. %We observe that not only does LOKI achieve the best final performance, but it also maintains superior performance over time and suffers less deterioration compared to competitors.
%\paragraph{Layer Selection Frequency.} 
We also study the frequency with which different layers are selected using LOKI during editing. \Cref{fig:freqs} shows layer choice frequencies for Llama-3 across 1000 sequential edits on the ZsRE dataset, where each layer's selection rate is shown as a percentage; layers used by AlphaEdit/MEMIT are depicted in red for reference. We observe that LOKI prefers to edit the later layers quite heavily and sometimes chooses the mid-to-late layers.

\paragraph{Null-space Overlap  and Feature Preservation.} We empirically study the overlap between the weight null-space and the feature null-space as in AlphaEdit~\cite{fang2024alphaedit}. We measure this overlap using the Grassmannian distance between subspaces, and comparing the overlap between the weight null-space and the feature null-space, to that of a random subspace and the feature null-space. In \cref{tab:grassmann}, we observe that the weight null-space  significantly  overlaps with  the feature null-space, validating {our claims} in \cref{thm:null}.
%
%\paragraph{Feature preservation.} 
We also conduct a qualitative study into the model features resulting from LOKI. In \cref{fig:embeds}, we visualize the last token of the Llama-3 model embeddings for the ZsRE dataset before and after 1000 sequential edits using TSNE. We see that similar to AlphaEdit, LOKI is able to avoid the heavy shift in model features observed using FT, without having to compute the feature null-space (\cref{eq:alphanull}) as AlphaEdit does. This further validates our intuition about the effectiveness of weight null-space projection for preserving past knowledge.

% \begin{figure}[h!]
%     \centering
%     \includegraphics[width=0.49\linewidth]{figs/embeds.pdf}
%     \caption{TSNE visualization of Llama-3-8B-Instruct final layer embeddings before and after 1000 sequential edits on the ZsRE dataset. LOKI manages to avoid a significant shift in model embeddings without access to previous knowledge.}
%     \label{fig:embeds}
% \end{figure}

% \begin{figure}[h!]
%     \centering
%     \includegraphics[width=0.49\linewidth]{figs/freqs.pdf}
%     \caption{Visualization of the layer selection frequency for Llama-3-8B-Instruct during 1000 sequential edits. The horizontal axis represents each layer, and the vertical axis shows what percentage of edits chose that layer. Unlike AlphaEdit/MEMIT, LOKI chooses editing layers per-sample, and prefers the mid-to-late layers.}
%     \label{fig:freqs}
% \end{figure}

% \begin{figure}[h!]
%     \centering
%     \includegraphics[width=0.49\linewidth]{figs/overtime.pdf}
%     \caption{Average accuracy of various methods every 100 edits for 1000 sequential edits on the  Llama-3-8B-Instruct and ZsRE dataset. LOKI suffers the least performance degradation over time across competitors.}
%     \label{fig:overtime}
% \end{figure}

% Side-by-side figures
\begin{figure}[!t]
\centering
\begin{minipage}{.49\textwidth}
  \centering
  \includegraphics[width=\linewidth]{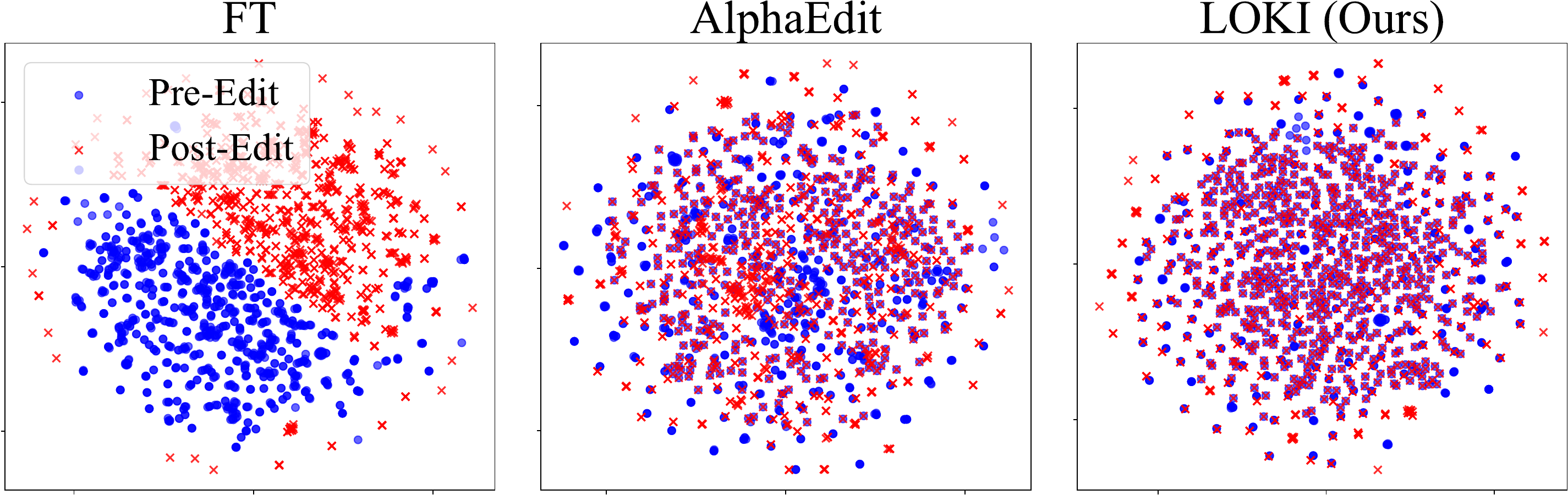}
  \captionof{figure}{TSNE visualization of Llama-3-8B-Instruct final layer embeddings before and after 1000 sequential edits on the ZsRE dataset. LOKI manages to avoid a significant shift in model embeddings without access to previous knowledge.}
  \label{fig:embeds}
\end{minipage}%
\hfill
\begin{minipage}{.49\textwidth}
  \centering
  \tiny
  \begin{tabular}{c|c|c|c|c|c|c}
    \toprule
       \multicolumn{2}{c|}{Layer}  &  4 & 5 & 6 & 7 & 8\\
       \midrule

        \multirow{2}{*}{Llama-3}& W-Null (Ours) & \textbf{2.42} &\textbf{3.19} & 
        \textbf{2.56} & \textbf{2.40} & \textbf{2.42}\\
        & Random & 4.28 & 5.01 & 4.18 & 3.69 & 4.82\\
        \midrule
        \multirow{2}{*}{Mistral-7B}& W-Null (Ours) & \textbf{2.88} &\textbf{2.85} & 
        \textbf{2.56} & \textbf{2.41} & \textbf{2.79}\\
        & Random & 4.68 & 4.43 & 4.10 & 3.59 & 4.33 \\
        \bottomrule
    \end{tabular}
  \captionof{table}{Grassmannian distance between the null-space of the weights (our proposed method) and the feature null-space (as used by AlphaEdit) compared to a random subspace of the same dimension. The weight null-space has clearly lower distance to the feature null-space.}
  \label{tab:grassmann}
\end{minipage}
\end{figure}

\paragraph{Ablation Study.}
We perform an ablation study of all major components of LOKI: the weight null-space projection (Null), the dynamic layer selection (Layer), and the added HSIC bottleneck regularization loss term (Reg.). When null-space projection is absent, we don't project the loss gradients during optimization. When the layer selection is absent, we always optimize the same layers as AlphaEdit, and when the regularization term is absent, we do not add it to the optimization loss function. \Cref{tab:ablation} shows the results on the ZsRE dataset for 1000 sequential edits on the Llama-3-8B-Instruct architecture. We observe that all three components have an overall positive effect on average performance when added, with the effect of the weight null-space projection being especially high, followed by the layer selection, and finally a marginal improvement with the HSIC regularization term. This further validates our intuition about the weight null-space projection, as well as the design of our layer selection criterion. We experiment further on the effects of layer selection in App.~\ref{app:combining}.

\begin{figure}[!t]
\centering
\begin{minipage}{.49\textwidth}
  \centering
  \tiny
    
  \includegraphics[width=\linewidth]{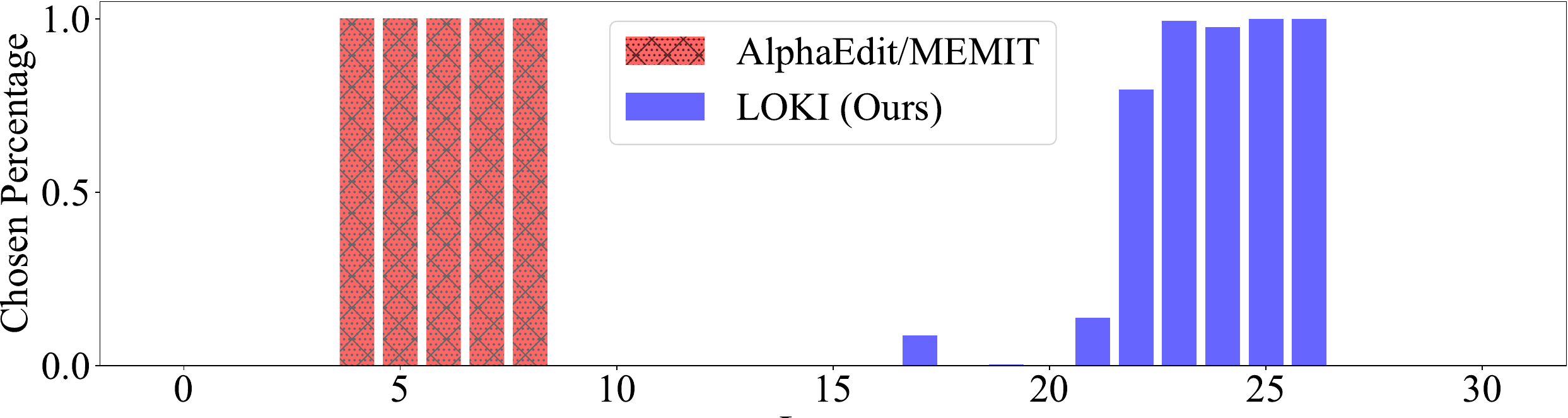}
  \captionof{figure}{Visualization of the layer selection frequency for Llama-3 during 1000 sequential edits. The horizontal axis represents layers, and the vertical axis shows what percentage of edits chose that layer. Unlike AlphaEdit/MEMIT, LOKI chooses editing layers per-sample, and prefers the mid-to-late layers.}
  \label{fig:freqs}
\end{minipage}%
\hfill
\begin{minipage}{.49\textwidth}
  \centering
  \tiny
  \begin{tabular}{c|c|c|c|c|c|c}
\toprule
Reg. & Layer & Null & Rel. ($\uparrow$)  & Gen. ($\uparrow$)  & Loc. ($\uparrow$)  & Avg. ($\uparrow$)  \\
\midrule
    &      &     & 0.66 & 0.62 & 0.2  & 0.5  \\
    &      & \cmark    & 0.78 & 0.63 & 0.85 & 0.76 \\
    & \cmark     &     & 0.82 & 0.79 & 0.41 & 0.68 \\
    & \cmark     & \cmark    & \textbf{0.91} & \textbf{0.87} & 0.87 & 0.88 \\
\cmark    &      &     & 0.71 & 0.67 & 0.26 & 0.55 \\
\cmark    &      & \cmark    & 0.81 & 0.65 & 0.86 & 0.77 \\
\cmark    & \cmark     &     & 0.83 & 0.82 & 0.49 & 0.71 \\
\midrule
\cmark    & \cmark     & \cmark    & \textbf{0.91} & \textbf{0.87 }& \textbf{0.88} & \textbf{0.89}\\
\bottomrule
\end{tabular}
  \captionof{table}{Ablation study of LOKI components on the ZsRE dataset ($\tau=1,000$, Llama). Metrics include Reliability (Rel.~$\uparrow$), Generalization (Gen.~$\uparrow$), Locality (Loc.~$\uparrow$), and their mean (Avg.~$\uparrow$). \textbf{Bold} is the best result. The full configuration achieves the highest scores.}
  \label{tab:ablation}
\end{minipage}
\end{figure}
% \begin{table}[h!]
% \centering
% \scriptsize
% \label{tab:ablation}
% \caption{Ablation study of LOKI components on the ZsRE dataset ($\tau=1,000$, Llama-3-8B-Instruct). Metrics include Reliability (Rel.), Generalization (Gen.), and Locality (Loc.), with Avg. representing their mean. Bold indicates the best result for each metric. The results demonstrate that each component contributes to the final performance, with the full configuration achieving the highest scores across all metrics.}
% \begin{tabular}{c|c|c|c|c|c|c}
% \toprule
% HSIC & Layer & Null & Rel  & Gen  & Loc  & Avg  \\
% \midrule
%     &      &     & 0.66 & 0.62 & 0.2  & 0.5  \\
%     &      & \cmark    & 0.78 & 0.63 & 0.85 & 0.76 \\
%     & \cmark     &     & 0.82 & 0.79 & 0.41 & 0.68 \\
%     & \cmark     & \cmark    & \textbf{0.91} & \textbf{0.87} & 0.87 & 0.88 \\
% \cmark    &      &     & 0.71 & 0.67 & 0.26 & 0.55 \\
% \cmark    &      & \cmark    & 0.81 & 0.65 & 0.86 & 0.77 \\
% \cmark    & \cmark     &     & 0.83 & 0.82 & 0.49 & 0.71 \\
% \midrule
% \cmark    & \cmark     & \cmark    & \textbf{0.91} & \textbf{0.87 }& \textbf{0.88} & \textbf{0.89}\\
% \bottomrule
% \end{tabular}
% \end{table}

\section{Conclusion}

The major challenge in lifelong knowledge editing is to edit knowledge efficiently and sequentially, without relying on external memory, additional parameters, or extensive preprocessing. We attempt to address this challenge by proposing LOKI, a novel knowledge editing method that directly modifies model weights without auxiliary parameters. LOKI performs per-sample layer selection using an information bottleneck criterion and utilizes the model weight null-space to avoid requiring access to previous knowledge samples and features. We empirically validate the performance of LOKI in a variety of scenarios, observing up to a $14\%$ increase in performance compared to existing methods. Additionally, we perform exploratory experiments and ablation studies, validating our intuitions behind our design choices. Finally, we discuss limitations and future work in App.~\ref{app:limit}.

% \section*{Impact Statement}
% In this work, we present LOKI, a novel knowledge editing method that aims to facilitate the incorporation of new knowledge into LLMs. As LLMs see widespread use in society, knowledge editing in LLMs can be beneficial in a variety of aspects: fixing model mistakes~\cite{balachandran2022correcting}, correcting hallucinations~\cite{ji2023survey}, or removing bias~\cite{ferrara2023should}. However, as with any algorithm capable of modifying LLM behaviors, knowledge editing methods can be misused. For example, models can be edited to encourage toxic behavior, generate misinformation, or use inappropriate language. These risks are not exclusive to LOKI, and we believe the behavior and usage of LLMs should be sufficiently regulated to prevent misuse, including with knowledge editing. Overall, we believe that the benefits of LOKI are substantial enough to outweigh potential risks and that LOKI does not raise any new or exclusive ethical concerns.
% \section*{References}
\section*{Acknowledgments}
The authors gratefully acknowledge support from the National Science Foundation through grants 2414652 and 2112471.

\bibliography{example_paper}
\bibliographystyle{plain}

%%%%%%%%%%%%%%%%%%%%%%%%%%%%%%%%%%%%%%%%%%%%%%%%%%%%%%%%%%%%
\newpage
\appendix

\section{HSIC Definition}
\label{app:hsic}

Let $\mathcal{F}$ and $\mathcal{G}$ be reproducing kernel Hilbert spaces (RKHSs) 
on domains $\mathcal{X}$ and $\mathcal{Y}$ with associated kernels $k: \mathcal{X} 
\times \mathcal{X} \rightarrow \mathbb{R}$ and $l: \mathcal{Y} \times \mathcal{Y} 
\rightarrow \mathbb{R}$. The Hilbert-Schmidt Independence Criterion (HSIC) between 
random variables $X$ and $Y$ with joint distribution $p_{XY}$ is defined as the 
squared Hilbert-Schmidt norm of the cross-covariance operator $C_{XY}$:
\begin{align}
    \mathtt{HSIC}(X, Y) := \|C_{XY}\|^2_{HS}
\end{align}

Given $m$ i.i.d. samples $\{(x_i, y_i)\}_{i=1}^m$ from $p_{XY}$, the empirical 
estimator is:
\begin{align}
    \widehat{\mathtt{HSIC}}(X, Y) = \frac{1}{(m-1)^2} \mathrm{tr}(KHLH)
\end{align}
where $K_{ij} = k(x_i, x_j)$, $L_{ij} = l(y_i, y_j)$ are the Gram matrices, and 
$H = I - \frac{1}{m}\mathbf{1}\mathbf{1}^\top$ is the centering matrix. With 
universal kernels (e.g., Gaussian), $\mathtt{HSIC}(X, Y) = 0$ if and only if $X$ 
and $Y$ are statistically independent~\cite{gretton2005hsic}.

\section{Proofs}

\WeightNullspace*
\label{app:proof}
\begin{proof}
For the sake of brevity, we will omit the $l$ superscript and $0$ subscript. Let $\hat{z}_{i,1}, ..., \hat{z}_{i,k}$ be the vectors corresponding to the $i$-th sample. Let $z_{i,1},...,z_{i,k}$ be the corresponding output tokens of $W_{\mathtt{out}}^{(l)}$. Assuming loss function $\mathcal{L}_i$, and learning rate $\alpha_i$, the gradient update to $W_{\mathtt{out}}$ is as follows due to the chain rule, with a slight abuse of notation:
\begin{align}
    \frac{\partial\mathcal{L}_i}{\partial W_{\mathtt{out}}} = \sum_{j=1}^{k_i} \frac{\partial\mathcal{L}_i}{\partial{z_{i,j}}}\cdot\frac{\partial z_{i,j}}{\partial W_{\mathtt{out}}} = \sum_{j=1}^{k_i} \frac{\partial\mathcal{L}_i}{\partial{z_{i,j}}}\cdot\frac{\partial W_{\mathtt{out}}^{(l)} \hat{z}_{i,j}}{\partial W_{\mathtt{out}}^{(l)}}  = \sum_{j=1}^{k_i}\frac{\partial\mathcal{L}_i}{\partial z_{i,j}} \hat{z}_{i,j}^\top
\end{align}
Therefore, we can write $W_{\mathtt{out}}$ as the culmination of all the gradient updates $1,...,n$, as
\begin{align}
    W_{\mathtt{out}} = W_{0} + \sum_{i=1}^n\alpha_i \cdot \sum_{j=1}^{k_i}\frac{\partial\mathcal{L}_i}{\partial z_{i,j}} \hat{z}_{i,j}^\top = W_{0} + \sum_{i=1}^n\sum_{j=1}^{k_i}\alpha_i \cdot \frac{\partial\mathcal{L}_i}{\partial z_{i,j}} \hat{z}_{i,j}^\top 
\end{align}
Where $W_0$ is the initialization. We can take the RHS to be a single $n\times k$ sum, and write it in matrix form:
\begin{align}
    W_{\mathtt{out}} = W_0 + (\partial Z) A \hat{Z}^\top
\end{align}
where $A \in \mathbb{R}^{s\times s}$ is a diagonal matrix, with elements $\alpha_i$ repeated $k_i$ times each with $s:=\sum k_i$, $\partial Z\in \mathbb{R}^{d_1 \times s}$ is a matrix with $\frac{\partial\mathcal{L}_i}{\partial z_{i,j}} $ as its columns, and $\hat{Z} \in \mathbb{R}^{d_2\times s}$ is a matrix with $\hat{z}_{i,j}$ as its columns. Now assuming the initialization $W_0$ is small or zero, i.e., $W_0 \approx 0$, we can write:
$$W_{\mathtt{out}} = (\partial Z) A \hat{Z}^\top$$
which completes the proof, since  $\hat{Z}^\top$ is the rightmost term of the product, and any vector $\mathbf{v}$ such that $\hat{Z}^\top \mathbf{v} = \mathbf{0}$ (i.e., $\mathbf{v} \in \mathtt{Null}(\hat{Z}^\top)$) must also satisfy $W_{\mathtt{out}} \mathbf{v} = (\partial Z) A (\hat{Z}^\top \mathbf{v}) = \mathbf{0}$. Therefore, the null space of $\hat{Z}^\top$ is necessarily contained within the null space of $W_{\mathtt{out}}$, satisfying $\mathtt{Null}(\hat{Z}^\top) \subset \mathtt{Null}(W_{\mathtt{out}})$.

\end{proof} 

%%%%%%%%%%%%%%%%%%%%%%%%%%%%%%%%%%%%%%%%%%%%%%%%%%%%%%%%%%%%%%%%%%%%%%%%%%%%%%%

\section{Implementation Details}
\label{app:impdetails}

We base our implementation on the EasyEdit framework~\cite{wang2023easyedit}, available at \url{https://github.com/zjunlp/EasyEdit}, the code for our method is available in the supplementary material under the LOKI folder(s) in the $\mathtt{easyeditor/models}$, and $\mathtt{hparams}$ directories, and the code to run our experiments can be found in $\mathtt{run\_zsre\_llama2.py}$ and $\mathtt{run\_wise\_editing.py}$ files in the $\mathtt{examples}$ folder. We will make this code public upon acceptance. All experiments are executed on NVIDIA H200 GPUs using three model architectures: Llama-3-8B-Instruct, Mistral-7B-v0.1, and GPT-J-6B.
For all experiments, we perform edits sequentially with a batch size of 1, ensuring the model is instantly updated and corrected after each input following the lifelong model editing setting~\cite{hartvigsen2023aging, wang2024wise}.

To mitigate the variance introduced by individual edit samples in experiments involving a small number of sequential edits, we report results averaged over multiple independent runs. In Table~\ref{tab:main}, the total number of sequential edits $\tau$ increases from 1 to 1,000 for the ZsRE dataset, and from 1 to 600 for the SelfCheckGPT dataset in Table~\ref{tab:hallucination}. For experiments involving a small number of edits (typically from $\tau = 1$ up to $\tau = 100$), we report results averaged over multiple runs using distinct edit samples. For example, results for $\tau = 1$ on the ZsRE dataset are averaged over 1,000 independent experiments, each using a different edit sample. Similarly, for $\tau = 10$ on the SelfCheckGPT dataset, we average results over 60 runs with distinct edit samples. For the Temporal dataset (Tables~\ref{tab:ood_generalization} and \ref{tab:ood_generalization_sd}), results for $\tau=10$ are averaged over 7 independent experiments. These averaging procedures are consistent across the main results and the extended results in App.~\ref{app:extended}.

\subsection{Edit time comparison}

\begin{figure}[tbp]
    \centering
    \includegraphics[width=\linewidth]{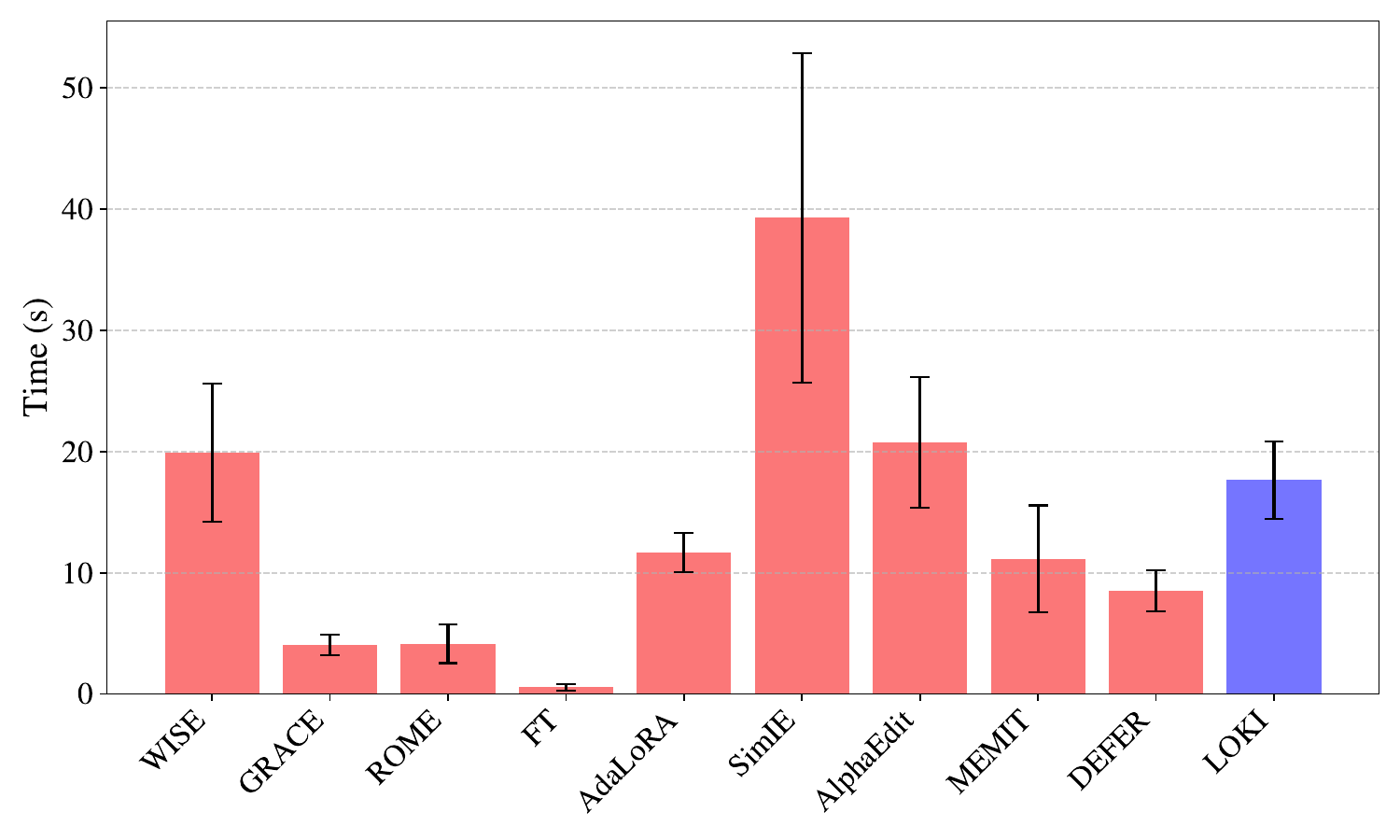}
    \caption{Average edit time per method measured on an NVIDIA H200 GPU. We report the average time per edit with standard deviation across 600 edits on the SelfCheckGPT dataset. LOKI (blue bar) is faster than WISE and AlphaEdit, the most competitive baselines, while incurring some overhead compared to simpler methods. This overhead stems from LOKI's dynamic layer selection process and the computation of the null space for the $W_{out}$ matrix during each edit. Note that, unlike MEMIT and AlphaEdit, LOKI requires no preprocessing; the runtimes shown for those methods exclude their one-time covariance (and, for AlphaEdit, null-space projection) computation, which would substantially increase their effective cost (see App.~\ref{app:impdetails}).}
    \label{fig:edit_time_comparison}
\end{figure}

Figure~\ref{fig:edit_time_comparison} illustrates the edit execution times. LOKI averages $17.6$ seconds per edit, outperforming WISE and AlphaEdit in speed but slower than the other baselines; these are the mean durations of 600 individual edits on Llama-3-8B-Instruct with the SelfCheckGPT dataset. The runtimes plotted for MEMIT and AlphaEdit \textbf{do not} include the amortized one-time cost of gathering past-knowledge covariance statistics (and, for AlphaEdit, of building null-space projections $P$ on top of them), so fair comparison would shift both methods much slower once preprocessing is counted. 

Table~\ref{tab:covariance_build} gives estimated wall-clock for that preprocessing alone---Wikipedia, 100k samples, Llama-3-8B-Instruct---on H200 and A100: across all 32 layers the EasyEdit~\cite{wang2023easyedit} pipeline we use totals about $\sim$25\,h (H200) and $\sim$56\,h (A100), with a layer-wise breakdown in the table. There, statistics are collected \emph{separately for each target layer} (the corpus is iterated per layer), each pass performing a \emph{partial forward} from input embeddings through the transformer only as far as the rewrite module for that layer; the forward ends right after that module so the network tail is not executed, yet deeper rewrite modules still require more compute per batch than earlier ones. AlphaEdit repeats this covariance accumulation and adds the per-layer projection $P$, whose incremental cost is noted in the table caption. Overall, LOKI avoids this costly preprocessing: it does not iterate a large corpus to build layer-wise covariance or feature null-spaces from past activations, and instead obtains its projections directly from the model weights at edit time. When this preprocessing is amortized across 600 edits, both methods remain substantially slower than LOKI: AlphaEdit is $1.94\times$/$9.75\times$ slower on H200 (only 5 edited layers/all 32) and $2.96\times$/$20.57\times$ slower on A100, while MEMIT is $1.38\times$/$9.10\times$ slower on H200 and $2.38\times$/$19.82\times$ slower on A100.

\begin{table}[!t]
\centering
\caption{Covariance matrix build times on H200 vs.\ A100 (Llama-3-8B-Instruct, Wikipedia, 100k samples). MEMIT and AlphaEdit both require these covariance statistics; AlphaEdit additionally builds a null-space projection matrix $P$ per layer, adding $\sim$$30$--$37$\,s/layer on H200 and $\sim$$60$--$72$\,s/layer on A100 ($\sim$$16$--$20$\,min and $\sim$$34$--$38$\,min total for all 32 layers, respectively).}
\label{tab:covariance_build}
\begin{tabular}{lrr}
\toprule
\textbf{Layer} & \textbf{Est.\ full (min) H200} & \textbf{Est.\ full (min) A100} \\
\midrule
0  & 13.3 & 34.2  \\
4  & 21.8 & 52.4  \\
9  & 32.5 & 75.1  \\
13 & 41.2 & 93.6  \\
18 & 51.9 & 116.6 \\
22 & 60.5 & 135.2 \\
27 & 71.3 & 159.5 \\
31 & 80.2 & 177.8 \\
\midrule
\textbf{All 32} & \textbf{$\sim$1{,}490 ($\sim$25 h)} & \textbf{$\sim$3{,}375 ($\sim$56 h)} \\
\bottomrule
\end{tabular}
\end{table}

\subsection{Implementation of LOKI}

LOKI performs knowledge editing through a multi-stage process. First, it dynamically selects critical layers using Hilbert-Schmidt Independence Criterion (HSIC) to identify layers with maximal information flow between input and output activations. The method then computes null-space projection matrices for each selected layer directly from the weight matrices via singular value decomposition, ensuring new edits are projected onto subspaces orthogonal to the weight matrix structure. Weight updates are optimized via Adam with a combined loss function that includes: (1) cross-entropy loss on target tokens, (2) KL divergence regularization to preserve model behavior on unrelated inputs, and (3) HSIC-based bottleneck constraints to maintain information flow through intermediate layers.

Below are several implementation details regarding the different components of LOKI:
\begin{itemize}
    \item As mentioned above, we use an additional KL divergence loss term similar to that of~\cite{meng2022mass} to preserve the KL divergence of model responses to the prompt "subject is a \{insert subject\}", in order to regularize the model further. The $kl\_factor$ hyperparameter is the coefficient of this loss term, added to the objective in \cref{eq:finalopt}.
    \item Additionally, similar to~\cite{meng2022locating} once again, we prepend each input $X_t$ with several (typically 5) prompts generated by the model from scratch, to improve generalization. The resulting inputs are then treated as a batch and used in the loss terms instead of $X_t$ by itself. That is $\bar{X}_t = [[G_1, X_t], [G_2,X_t], ...]$, where $G_i$s are generated by the LLM from scratch.
    \item The above generalization also serves as samples for the HSIC criterion. As HSIC measures the independence between distributions, it requires multiple samples from each distribution. The embeddings of each input prepended by a prompt are averaged over the dimension of the tokens and each used as one sample for the HSIC criterion in \cref{eq:hib}.
    %\item We empirically find that the null-space matrix in \cref{eq:lokinull} is more effective at reducing interference if it is low-rank. This is not surprising since it aligns with the fact that the desired subspace is a subset of the weight null-space, and that low-rank updates in general tend to interfere less with past knowledge, as is exploited in LoRA fine-tuning. Therefore, we use the first 1000 (represented by the $null\_dim$ hyperparameter) bases of the zero singular value bases of the $V^{(l)}$ matrix for calculating $P^{(l)}$, that is: $P^{(l)} = \bar{V}^{(l)}[:1000]\bar{V}^{(l)}[:1000]^\top$
    \item We empirically find that the null-space matrix in \cref{eq:lokinull} tends to be more effective at reducing interference if it is low-rank. This is not surprising since it aligns with the fact that the desired subspace is a subset of the weight null-space, and that low-rank updates in general tend to interfere less with past knowledge, as is exploited in LoRA fine-tuning. Therefore, we retain the first $null\_dim$ bases of the zero singular value bases of the $V^{(l)}$ matrix for calculating $P^{(l)}$, that is: $P^{(l)} = \bar{V}^{(l)}[:,:null\_dim]\bar{V}^{(l)}[:,:null\_dim]^\top$, where $null\_dim$ is task-specific (\cref{tab:loki_hyperparams}).
    \item The HSIC regularization term added to the objective in \cref{eq:finalopt} is slightly modified to use the output of a linear transform $Z$ rather than previously $\hat{Z}$ so that both terms are differentiable w.r.t. to the modified weights. (Note that the layer selection term is exactly as described previously, and only the regularization term is changed slightly.) 
\end{itemize}

%Table~\ref{tab:loki_hyperparams} summarizes the hyperparameters used across all three models. Hyperparameter exploration was conducted for Llama-3-8B-Instruct on the ZsRE dataset. Other models were not further explored, which suggests that a slight improvement could be achieved through hyperparameter exploration. For all models, we dynamically select 3 layers with 25 optimization steps. The null-space dimension is set to 1000 and HSIC regularization coefficients $h_x$ and $h_y$ are 0.001. All models use a weight learning rate of $1 \times 10^{-4}$. 

We determine LOKI hyperparameters in three stages. For ZsRE (Tables~\ref{tab:main} and~\ref{tab:main_sd}), we search over candidate settings on Llama-3-8B-Instruct and deploy the best-performing configuration for both Llama and Mistral without additional per-model tuning. For SelfCheckGPT (Table~\ref{tab:hallucination} and~\ref{tab:hallucination_sd})) and Temporal (Table~\ref{tab:ood_generalization}and~\ref{tab:ood_generalization_sd})), we use the standard training and evaluation splits. Specifically, we run a bounded search on the training split, select the configuration that optimizes the in-split objective for each backbone, and evaluate that single configuration on the held-out split used in the tables. Table~\ref{tab:loki_hyperparams} collects every numerical setting that enters the reported LOKI runs.

% \begin{table}[htbp]
% \centering
% \caption{Default LOKI hyperparameters for ZsRE (Llama-3-8B-Instruct and shared settings for other models in Table~\ref{tab:main} unless task-specific values above apply).}
% \label{tab:loki_hyperparams}
% \begin{tabular}{lc}
% \hline
% \textbf{Hyperparameter} & \textbf{Parameter Value}\\ %\textbf{Llama-3-8B} & \textbf{Mistral-7B} & \textbf{GPT-J-6B} \\
% \hline
% Number of updated Layers & 3 \\ %& 3 & 5 \\
% Optimization steps ($num\_steps$) & 25 \\ %& 25 & 30 \\
% Weight learning rate ($w\_lr$) & $1 \times 10^{-4}$ \\ %& $1 \times 10^{-4}$ & $5 \times 10^{-5}$ \\
% Null-space dimension ($null\_dim$) & 1000 \\ %& 1000 & 1500 \\
% HSIC coefficient $h_x$ & 0.001 \\ %& 0.001 & 0.001 \\
% HSIC coefficient $h_y$ & 0.001 \\ %& 0.001 & 0.005 \\
% HSIC bandwidth ($hsigma$) & 1.0 \\ %& 1.0 & 1.0 \\
% KL regularization ($kl\_factor$) & 0.02 \\ %& 0.02 & 0.05 \\
% Norm constraint ($norm\_constraint$) & 0.05 \\ %& 0.05 & 0.01 \\
% %L1 regularization ($l1\_lambda$) & 0.01 \\ %& 0.01 & 0.01 \\
% %L2 regularization ($L2$) & 10 \\ %& 10 & 10 \\
% \hline
% \end{tabular}
% \end{table}

\begin{table}[!t]
\centering
\footnotesize
\setlength{\tabcolsep}{3pt}
\caption{LOKI hyperparameters for each benchmark and backbone reported in the main results. The ZsRE column is shared by Llama-3-8B-Instruct and Mistral-7B-v0.1; SelfCheckGPT lists separate settings for each model; Temporal lists GPT-J-6B.}
\label{tab:loki_hyperparams}
\begin{tabular}{@{}p{4.35cm}cccc@{}}
\toprule
 & \textbf{ZsRE} & \multicolumn{2}{c}{\textbf{SelfCheckGPT}} & \textbf{Temporal} \\
\cmidrule(lr){2-2} \cmidrule(lr){3-4} \cmidrule(lr){5-5}
\textbf{Hyperparameter} & Llama / Mistral & Llama-3-8B-Instruct & Mistral-7B & GPT-J-6B \\
\midrule
Number of updated Layers & 3 & 7 & 5 & 3 \\
Optimization steps ($num\_steps$) & 25 & 25 & 25 & 30 \\
Weight learning rate ($w\_lr$) & $1 \times 10^{-4}$ & $1.34 \times 10^{-4}$ & $2.58 \times 10^{-4}$ & $5.13 \times 10^{-5}$ \\
Null-space dimension ($null\_dim$) & 1000 & 1500 & 3000 & 1000 \\
HSIC coefficient $\lambda_x$ & 0.001 & $6.2 \times 10^{-4}$ & $4.46 \times 10^{-4}$ & $3.19 \times 10^{-4}$ \\
HSIC coefficient $\lambda_y$ & 0.001 & $7.05 \times 10^{-5}$ & $1.78 \times 10^{-4}$ & $6.62 \times 10^{-3}$ \\
HSIC bandwidth ($hsigma$) & 1.0 & 2.28 & 2.66 & 10.86 \\
KL regularization ($kl\_factor$) & 0.02 & 0.01 & 0.01 & 0.019 \\
Norm constraint ($norm\_constraint$) & 0.05 & $1.62 \times 10^{-1}$ & $1.37 \times 10^{-2}$ & $1.47 \times 10^{-3}$ \\
\bottomrule
\end{tabular}
\end{table}

\subsection{Competitors Hyperparameters}

\paragraph{FT.~\cite{zhang2024comprehensive}} For Llama-3-8B-Instruct and Mistral-7B-v0.1, we edit layer 21 with 50 training steps using Adam optimizer with learning rate $5 \times 10^{-4}$. For GPT-J-6B, we use layer 21 with 25 steps and the same learning rate. All models use a maximum sequence length 40.

\paragraph{AdaLoRA.~\cite{zhang2023adaptive}} We present results only on the ZsRE dataset, where we fine-tune Adaptive Low-Rank Adaptation (AdaLoRA) modules on \texttt{q\_proj} and \texttt{v\_proj}, following the default method in EasyEdit~\cite{wang2023easyedit}. We use maximum rank $r{=}8$, scaling $\alpha{=}32$, dropout $0.1$, Adam learning rate $1 {\times} 10^{-4}$, 50 optimizer steps per edit, and no weight decay.

\paragraph{ROME.~\cite{meng2022locating}} For Llama-3-8B-Instruct and Mistral-7B-v0.1, we edit layer 5 following the default settings. For GPT-J-6B, we also edit layer 5 following their implementation. All models use 25 gradient steps (20 for GPT-J-6B) with learning rate $5 \times 10^{-1}$ to compute hidden representations, and compute covariance statistics from 100,000 Wikipedia samples in float32 precision.

\paragraph{MEMIT.~\cite{meng2022mass}} For Llama-3-8B-Instruct and Mistral-7B-v0.1, we update layers [4, 5, 6, 7, 8] with covariance adjustment factor $\lambda = 15,000$. For GPT-J-6B, we target layers [3, 4, 5, 6, 7, 8] with the same $\lambda$ value. All models use 25 gradient steps with learning rate $5 \times 10^{-1}$, and 100,000 Wikipedia samples for covariance computation.

\paragraph{DEFER.~\cite{hartvigsen2023aging}} We follow the hyperparameter settings from~\cite{wang2024wise}, using learning rate $7 \times 10^{-5}$, 100 training steps per edit, and 0.5 threshold. For Llama-3-8B-Instruct and Mistral-7B-v0.1, we edit layer 27; and for GPT-J-6B, layer 25, following the implementation in~\cite{wang2023easyedit}.

\paragraph{GRACE.~\cite{hartvigsen2023aging}} Consistent with the original work, we set the learning rate to 1.0 and use the replace\_last strategy, modifying only the final token's activations in autoregressive settings. The deferral radius $\epsilon$ is set to 1.0 for all models. We use 50 training iterations for Llama-3-8B-Instruct and Mistral-7B-v0.1, and 200 for GPT-J-6B. The edited layers are 29 for Llama-3-8B-Instruct, 29 for Mistral-7B-v0.1, and 25 for GPT-J-6B.

\paragraph{WISE.~\cite{wang2024wise}} We adopt the hyperparameter settings from~\cite{wang2024wise}. For Llama-3-8B-Instruct, we use SGD optimizer with learning rate $\eta = 1.0$, mask ratio $\rho = 0.2$, activation margins $\alpha = 5.0$, $\beta = 20.0$, $\gamma = 10.0$, activation ratio 0.5, merge weight 0.5, and 30 training steps. For Mistral-7B-v0.1, we use the same configuration except 70 training steps, activation ratio 0.88, and merge weight 0.5. For GPT-J-6B, we use 70 steps, activation margins $\alpha = 15.0$, $\beta = 40.0$, $\gamma = 20.0$, activation ratio 0.88, and merge weight 1.0. All models use 2 knowledge shards and edit layer 29 (Llama-3), 27 (Mistral), or 21 (GPT-J-6B).

\paragraph{AlphaEdit.~\cite{fang2024alphaedit}} For Llama-3-8B-Instruct and Mistral-7B-v0.1, we edit critical layers [4, 5, 6, 7, 8] following~\cite{meng2022mass}. To compute hidden representations at each critical layer, we perform 25 optimization steps with learning rate $1 \times 10^{-1}$. For GPT-J-6B, we target layers [3, 4, 5, 6, 7, 8] with 25 optimization steps per layer and learning rate $5 \times 10^{-1}$. As in MEMIT, we use $\lambda = 15,000$ and 100,000 Wikipedia samples. The null-space threshold is $2 \times 10^{-2}$ for all models, and L2 regularization is set to 1 for Llama-3 and Mistral, and 10 for GPT-J-6B.

\paragraph{SimIE.~\cite{guo2025towards}} For Llama-3-8B-Instruct and Mistral-7B-v0.1 we use ROME as the base method with $\lambda = 10$ for the best performance as detailed in the original paper. The hyperparameters of ROME are as before.

%%%%%%%%%%%%%%%%%%%%%%%%%%%%%%%%%%%%%%%%%%%%%%%%%%%%%%%%%%%%%%%%%%%%%%%%%%%%%%%

\section{Experimental Details}
\label{app:expdetails}

\subsection{Datasets}

\paragraph{ZsRE.}
Zero-shot Relation Extraction (ZsRE)~\cite{levy2017zero} is a context-free question answering benchmark designed to evaluate factual knowledge storage and retrieval in language models. Each instance contains an edit prompt $x_e$ (e.g., \textit{``What was the date of Air France Flight 447?''}), a target answer $y_e$ (e.g., \textit{``1 June 2009''}), a semantically equivalent rephrase $x'_e$ (e.g., \textit{``What's the date of Air France Flight 447?''}) for generalization evaluation, and an unrelated locality prompt $x_{\text{loc}}$ (e.g., \textit{``where is the 7th game of the world series played''}) with its corresponding answer to assess whether edits affect irrelevant knowledge. We use the ZsRE splits provided by EasyEdit: 163,196 training samples, 1,037 edit samples, and 19,086 evaluation samples. For all ZsRE results reported in Table~\ref{tab:main}, we randomly sample 1,000 instances from the edit split.

\paragraph{SelfCheckGPT.}
We use the hallucination correction benchmark introduced by~\cite{hartvigsen2023aging}, constructed from the SelfCheckGPT dataset~\cite{manakul2023selfcheckgpt}. The original SelfCheckGPT work prompted GPT-3~\cite{brown2020language} to generate 238 Wikipedia-style biographies and annotated each sentence for factual accuracy. For model editing, hallucinated sentences are replaced with corresponding factually correct sentences from actual Wikipedia entries. Each instance contains a prompt $x_e$ consisting of GPT-3-generated text up to the hallucination (e.g., \textit{``This is a Wikipedia passage about ron meagher.''}) and a target $y_e$ containing the correct Wikipedia sentence (e.g., \textit{``Ron Meagher (born October 2, 1941, Oakland, California, USA) is best known as the bassist of the American rock band The Beau Brummels.''}). The dataset comprises 1,392 potential edits with substantially longer token sequences than standard QA benchmarks. Locality is evaluated using unrelated text continuations. Following~\cite{wang2024wise}, we use 600 test samples after filtering excessively long sequences.

\paragraph{Temporal.}
The Temporal dataset~\cite{hewitt2024model} evaluates out-of-distribution (OOD) generalization for emerging entities. Each instance contains a prefix $x_e$ extracted from an entity's Wikipedia page (e.g., \textit{``Luka Dončić''}), a GPT-4 generated description $y_e$ containing potentially noisy but informative content about the entity, and an actual Wikipedia suffix $y_{\text{ood}}$ for OOD evaluation. The dataset tests whether edits performed on formulaic examples generalize to complex natural text, where the distributional shift involves linguistic complexity rather than domain change. Locality samples are drawn from the Pile~\cite{gao2020pile}, the pretraining corpus of GPT-J.

\subsection{Metrics}

We evaluate editing methods using three standard metrics~\cite{yao2023editing, zhang2024comprehensive}:

\paragraph{Reliability (Rel.).}
Measures whether the edited model correctly produces the target output for edit prompts. For QA tasks, this is computed as the accuracy on edit samples:
\begin{align}
\text{Rel.} &= \frac{1}{\tau}\sum_{t=0}^{\tau} \mathbbm{1}\!\left(f_{\theta_t}(x_e^\top) = y_e^\top\right)\,.
\end{align}
Higher accuracy is better ($\uparrow$). For hallucination correction (SelfCheckGPT), we report perplexity (PPL) on the target text following~\cite{wang2024wise}; lower PPL is better ($\downarrow$).

\paragraph{Generalization (Gen.).}
Assesses whether edits transfer to semantically equivalent rephrasings of the edit prompt. When paraphrases $x'_{e}$ are available,
\begin{align}
\text{Gen.} &= \frac{1}{\tau}\sum_{t=0}^{\tau} \mathbbm{1}\!\left(f_{\theta_t}(x_{e'}^\top) = y_e^\top\right)\,.
\end{align}
Higher accuracy is better ($\uparrow$). This formulation applies where paraphrase prompts are supplied (ZsRE). For OOD evaluation on Temporal, we use a probability threshold where success is defined as $\mathcal{L}_{\theta_t}(x_e, y_{\text{ood}}) < -\log(0.8)$; higher success rate is better ($\uparrow$). For SelfCheckGPT we do not report Gen.\ because the dataset does not include paraphrase examples.

\paragraph{Locality (Loc.).}
Evaluates whether edits preserve the model's behavior on unrelated inputs:
\begin{align}
\text{Loc.} &= \frac{1}{\tau}\sum_{t=0}^{\tau} \mathbbm{1}\!\left(f_{\theta_t}(x_{\text{loc}}^\top) = f_{\theta_0}(x_{\text{loc}}^\top)\right)\,.
\end{align}
Higher agreement with the pretrained model's predictions on locality prompts is better ($\uparrow$). This metric is shared across all tasks.

\subsection{Competitors}

\paragraph{FT.~\cite{zhang2024comprehensive}} Fine-tunes a single MLP layer using the autoregressive loss while freezing all other parameters. It utilizes a cross-entropy loss on the target answer while masking the original text.

\paragraph{AdaLoRA.~\cite{zhang2023adaptive}} Freezes the base model and fits low-rank adapters with an adaptive rank budget across target modules. Training minimizes next-token cross-entropy on the new answer while masking the prompt region.

\paragraph{ROME.~\cite{meng2022locating}} Rank-One Model Editing locates factual associations in specific MLP layers via causal tracing, then modifies parameters through least-squares approximation to inject new knowledge.

\paragraph{MEMIT.~\cite{meng2022mass}} Mass-Editing Memory in a Transformer extends ROME to multiple layers simultaneously, enabling batch editing of thousands of facts in a single operation.

\paragraph{DEFER.~\cite{hartvigsen2023aging}} Re-implementation of SERAC~\cite{mitchell2022memory} for lifelong editing. Introduces a scope classifier $g$ that routes inputs to either the original model or an auxiliary one-layer counterfactual network $o$. Both networks are jointly fine-tuned during editing.

\paragraph{GRACE.~\cite{hartvigsen2023aging}} General Retrieval Adaptors for Continual Editing maintains a discrete key-value codebook without modifying model weights. During inference, input activations are matched against cached keys; if within a learned deferral radius $\epsilon$, the corresponding value replaces the layer output.

\paragraph{WISE.~\cite{wang2024wise}} Introduces a dual-memory architecture with a frozen main memory (original weights) and a trainable side memory for edits. A routing mechanism directs queries to the appropriate memory based on activation patterns. Knowledge sharding distributes edits across orthogonal parameter subspaces, later merged via TIES-Merge to prevent conflicts.

\paragraph{AlphaEdit.~\cite{fang2024alphaedit}} Extends locate-then-edit methods by projecting weight perturbations onto the null space of previously edited knowledge, ensuring new edits minimally interfere with prior updates.

\paragraph{SimIE.~\cite{guo2025towards}}  Formulates the ideal parameter shift as the minimum-norm solution to a linear system, constructed using the Moore-Penrose
inverse, and subsequently enables recursive updates by truncating the limiting expression of the Moore-Penrose inverse under two mild assumptions. 

\section{Extended Results}
\label{app:extended}As discussed in App.~\ref{app:expdetails}, results for smaller sequential edit scales ($\tau \leq 100$ for ZsRE, SelfCheckGPT and $\tau = 10$ for Temporal) are averaged over several runs with distinct edit samples to mitigate variance. In this section, we extend the results to include the standard deviation across multiple independent runs to provide a comprehensive view of the stability and reliability of each editing method. The following tables present these detailed metrics all three datasets ZsRE in \cref{tab:main_sd}, SelfCheckGPT in \cref{tab:hallucination_sd}, and Temporal in \cref{tab:ood_generalization_sd}.

% The results across both tasks demonstrate that LOKI maintains high Reliability (Rel.) and competitive generalization performance as the number of sequential edits $\tau$ increases. For ZsRE (\cref{tab:main_sd}
% In the hallucination correction task (Table~\ref{tab:hallucination_sd}), LOKI consistently outperforms traditional methods like ROME and MEMIT, which exhibit catastrophic reliability failure at higher edit scales, particularly on Mistral-7B. While memory-based methods like GRACE achieve perfect Locality, they often suffer from significantly lower Reliability compared to LOKI. In the OOD generalization task (Table~\ref{tab:ood_generalization_sd}), LOKI displays robust OOD Generalization accuracy at $\tau=75$ on the Temporal dataset, whereas other methods like MEMIT experience a sharp decline in rewrite accuracy and generalization. Although LOKI exhibits a minor trade-off in Locality compared to WISE and AlphaEdit, its overall balance across the three metrics—Reliability, Generalization, and Locality—results in superior stability for long-term sequential editing.

\begin{table*}[!t]
%\caption{Results of editing two LLMs on the ZsRE dataset, with $\tau$ being the number of sequential edits. Avg represents the average of all three metrics. Higher is better for all metrics. LOKI outperforms existing methods consistently, especially so with a larger number of edits.}
\caption{Sequential model editing results on the ZsRE dataset (Llama-3-8B-Instruct and Mistral-7B). Metrics include Reliability (Rel. $\uparrow$), Generalization (Gen. $\uparrow$), and Locality (Loc. $\uparrow$) across $\tau$ sequential edits, with Avg. representing the mean of the three scores. \textbf{Bold} indicates the best result and underlining indicates the second-best result for each metric. LOKI consistently achieves the highest average performance across all edit scales, demonstrating superior scalability and stability compared to existing methods as $\tau$ increases up to 1,000.}
\centering
\scalebox{.45}{
\label{tab:main_sd}
\begin{tabular}{c|ccc|c|ccc|c|ccc|c|ccc|c}
\toprule
Method    & \multicolumn{4}{c|}{$\tau=$ 1}              & \multicolumn{4}{c|}{$\tau=$ 10}               & \multicolumn{4}{c|}{$\tau=$ 100}              & \multicolumn{4}{c}{$\tau=$ 1000}             \\
\midrule
          & Rel. ($\uparrow$)  & Gen. ($\uparrow$)    & Loc. ($\uparrow$)    & Avg. ($\uparrow$)         & Rel. ($\uparrow$)    & Gen. ($\uparrow$)    & Loc. ($\uparrow$)    & Avg. ($\uparrow$)         & Rel. ($\uparrow$)    & Gen. ($\uparrow$)    & Loc. ($\uparrow$)    & Avg. ($\uparrow$)         & Rel. ($\uparrow$)    & Gen. ($\uparrow$)    & Loc. ($\uparrow$)    & Avg. ($\uparrow$)         \\
          \midrule
\multicolumn{17}{c}{Llama-3-8B-Instruct}                                                                                                                                             \\
\midrule
FT        & \textbf{1}$_{\pm 0.00}$ & \textbf{0.99}$_{\pm 0.05}$  & 0.95$_{\pm 0.14}$   & \underline{0.98} $_{\pm 0.09}$       & 0.87$_{\pm 0.08}$   & 0.83$_{\pm 0.09}$   & 0.80$_{\pm 0.10}$   & 0.83$_{\pm 0.05}$        & 0.84$_{\pm 0.04}$   & 0.81 $_{\pm 0.05}$  & 0.42$_{\pm 0.06}$   & 0.69 $_{\pm 0.03}$       & 0.83   & 0.80   & 0.11   & 0.58        \\
%LoRA      & \textbf{1}$_{\pm 0.00}$ & 0.99$_{\pm 7.45}$  & 0.49$_{\pm 33.22}$   & 0.83 $_{\pm 11.42}$       & 0.76$_{\pm 14.54}$   & 0.72$_{\pm 13.90}$   & 0.32$_{\pm 15.00}$   & 0.60$_{\pm 11.88}$       & 0.47$_{\pm 7.15}$   & 0.45$_{\pm 7.32}$   & 0.30$_{\pm 8.49}$   & 0.41$_{\pm 5.62}$       & 0.36   & 0.35   & 0.20   & 0.31        \\
AdaLoRA   & \textbf{1}$_{\pm 0.00}$ & 0.73$_{\pm 0.30}$ & 0.82$_{\pm 0.24}$   & 0.85$_{\pm 0.13}$       & 0.83$_{\pm 0.08}$    & 0.60$_{\pm 0.11}$   & 0.57$_{\pm 0.14}$   & 0.67$_{\pm 0.07}$        & 0.54$_{\pm 0.04}$   & 0.49$_{\pm 0.04}$   & 0.44$_{\pm 0.07}$   & 0.49$_{\pm 0.04}$       & 0.36   & 0.35   & 0.26   & 0.32        \\
ROME      & \underline{0.99}$_{\pm0.03}$  & 0.96$_{\pm0.02}$  & 0.96$_{\pm0.01}$   & 0.97$_{\pm0.02}$        & 0.61$_{\pm0.11}$   & 0.59$_{\pm0.11}$   & 0.41$_{\pm0.20}$   & 0.54$_{\pm0.08}$        & 0.08$_{\pm0.07}$   & 0.08$_{\pm0.07}$   & 0.02$_{\pm0.01}$   & 0.06$_{\pm0.03}$        & 0.03   & 0.03   & 0.02   & 0.03        \\
MEMIT     & \underline{0.99}$_{\pm0.02}$  & 0.97$_{\pm0.02}$   & 0.98$_{\pm0.05}$   & 0.98$_{\pm0.04}$        & 0.68$_{\pm0.11}$   & 0.67$_{\pm0.11}$   & 0.73$_{\pm0.09}$   & 0.69$_{\pm0.06}$        & 0.03$_{\pm0.02}$   & 0.03$_{\pm0.01}$   & 0.01$_{\pm0.01}$   & 0.02$_{\pm0.03}$        & 0.00      & 0.00      & 0.00      & 0.00           \\
DEFER     & 0.82$_{\pm0.24}$ & 0.86$_{\pm0.22}$   & 0.68$_{\pm0.31}$   & 0.79$_{\pm0.15}$        & 0.66$_{\pm0.13}$   & 0.67$_{\pm0.13}$   & 0.45$_{\pm0.17}$   & 0.59$_{\pm0.08}$        & 0.33$_{\pm0.02}$   & 0.32$_{\pm0.02}$   & \textbf{1}$_{\pm 0.00}$      & 0.55$_{\pm0.01}$        & 0.31   & 0.31   & \textbf{1}      & 0.54        \\
GRACE     & \textbf{1}$_{\pm 0.00}$    & 0.46$_{\pm0.02}$   & \textbf{1}$_{\pm 0.00}$      & 0.82$_{\pm 0.00}$        & \textbf{1}$_{\pm 0.00}$      & 0.42$_{\pm0.15}$   & \textbf{1}$_{\pm 0.00}$     & 0.81$_{\pm0.03}$        & \textbf{1}$_{\pm 0.00}$      & 0.39$_{\pm0.07}$   & \textbf{1}$_{\pm 0.00}$     & 0.8$_{\pm0.02}$         & \textbf{1}      & 0.37   & \textbf{1}     & \underline{0.79}        \\
SimIE & 0.75$_{\pm0.28}$ & 0.71$_{\pm0.29}$ & 0.79$_{\pm0.26}$ & 0.75$_{\pm0.16}$ & 0.73$_{\pm0.09}$ & 0.69$_{\pm0.10}$ & 0.79$_{\pm0.09}$ & 0.73$_{\pm0.05}$ & 0.72$_{\pm0.02}$ & 0.67$_{\pm0.02}$ & 0.78$_{\pm0.04}$ & 0.72$_{\pm0.02}$ & 0.69 & 0.65 & 0.75 & 0.69\\
WISE      & \underline{0.99} $_{\pm 0.08}$ & \textbf{0.99} $_{\pm 0.08}$  & 0.92$_{\pm 0.24}$      & 0.97 $_{\pm 0.09}$       & 0.77$_{\pm 0.11}$   & 0.74$_{\pm 0.11}$   & \textbf{1} $_{\pm 0.00}$  & 0.84 $_{\pm 0.05}$       & 0.61 $_{\pm 0.05}$  & 0.6  $_{\pm 0.05}$  & \textbf{1} $_{\pm 0.00}$     & 0.73      $_{\pm 0.02}$  & 0.54   & 0.53   & 0.99      & 0.69        \\
AlphaEdit & \underline{0.99} $_{\pm 0.05}$ & 0.88 $_{\pm 0.25}$  & \underline{0.99}$_{\pm 0.04}$      & 0.95$_{\pm 0.09}$        & 0.92 $_{\pm 0.07}$  & \underline{0.84} $_{\pm 0.10}$  & \underline{0.98} $_{\pm 0.04}$  & \underline{0.92} $_{\pm 0.09}$       & \underline{0.90} $_{\pm 0.02}$  &\underline{0.84}$_{\pm 0.03}$   & 0.94 $_{\pm 0.04}$  & \underline{0.90}  $_{\pm 0.02}$      & 0.76   & \textbf{0.88}   & 0.65   & 0.76        \\
\midrule
\textbf{LOKI} (Ours)     & \textbf{1}$_{\pm 0.00}$    & \underline{0.98}$_{\pm 0.08}$  & \underline{0.99} $_{\pm 0.03}$ & \textbf{0.99}$_{\pm 0.06}$ & \underline{0.94} $_{\pm 0.04}$ & \textbf{0.87} $_{\pm 0.08}$ & \underline{0.98}$_{\pm 0.02}$ & \textbf{0.93}$_{\pm 0.03}$ & \textbf{0.93} $_{\pm 0.02}$ & \textbf{0.86} $_{\pm 0.03}$& \underline{0.96}$_{\pm 0.02}$ & \textbf{0.92} $_{\pm 0.01}$     & \underline{0.91} & \underline{0.87}  & \underline{0.88} & \textbf{0.89} \\
\midrule
\multicolumn{17}{c}{Mistral-7B}                                                                                                                                             \\
\midrule
FT        & \underline{0.99} $_{\pm 0.03}$ & \textbf{0.98} $_{\pm 0.06}$   & 0.75$_{\pm 0.29}$    & 0.91  $_{\pm 0.10}$       &  0.82 $_{\pm 0.12}$    & 0.77  $_{\pm 0.12}$   & 0.22 $_{\pm 0.18}$         & 0.60 $_{\pm 0.08}$   & 0.83  $_{\pm 0.03}$  & \underline{0.77}  $_{\pm 0.04}$   & 0.07 $_{\pm 0.02}$   & 0.56  $_{\pm 0.02}$       & 0.69   & 0.65   & 0.07   & 0.47         \\
%LoRA      & \textbf{1}$_{\pm 0.00}$ & 1$_{\pm 0.00}$   & 0.31$_{\pm 0.30}$    & 0.77 $_{\pm 0.10}$       & 0.63$_{\pm 18.11}$   & 0.60$_{\pm 17.34}$   & 0.21$_{\pm 16.07}$   & 0.48$_{\pm 14.77}$       & 0.33$_{\pm 8.66}$   & 0.32$_{\pm 7.42}$   & 0.17$_{\pm 10.61}$ & 0.28$_{\pm 7.79}$       & 0.15   & 0.15   & 0.07   & 0.12         \\
AdaLoRA   & \textbf{1}$_{\pm 0.00}$ & 0.91$_{\pm 0.18}$   & 0.84$_{\pm 0.22}$    & 0.92 $_{\pm 0.09}$        & 0.91$_{\pm 0.10}$    & 0.78$_{\pm 0.11}$   & 0.62$_{\pm 0.09}$    & 0.77$_{\pm 0.08}$        & 0.56$_{\pm 0.09}$   & 0.51$_{\pm 0.07}$   & 0.43$_{\pm 0.06}$  & 0.50$_{\pm 0.06}$       & 0.50   & 0.48   & 0.40   & 0.46         \\
ROME      & 0.79$_{\pm 0.23}$ & 0.77$_{\pm 0.23}$   & 0.98$_{\pm 0.05}$   & 0.85$_{\pm 0.11}$        & 0.58$_{\pm 0.12}$   & 0.57$_{\pm 0.12}$   & 0.75$_{\pm 0.19}$   & 0.63$_{\pm 0.08}$        & 0.05$_{\pm 0.03}$   & 0.05$_{\pm 0.03}$   & 0.02$_{\pm 0.02}$   & 0.04$_{\pm 0.01}$        & 0.04   & 0.04   & 0.02   & 0.03        \\
MEMIT     & 0.81$_{\pm 0.23}$ & 0.79$_{\pm 0.23}$   & \underline{0.99}$_{\pm 0.04}$   & 0.86$_{\pm 0.11}$        & 0.46$_{\pm 0.13}$   & 0.45$_{\pm 0.13}$   & 0.61$_{\pm 0.26}$   & 0.51$_{\pm 0.10}$        & 0.00 $_{\pm 0.00}$     & 0$_{\pm 0.00}$      & 0.01$_{\pm 0.00}$   & 0$_{\pm 0.00}$           & 0.04   & 0.04   & 0.02   & 0.03        \\
DEFER     & 0.59$_{\pm 0.27}$ & 0.68$_{\pm 0.29}$   & 0.79$_{\pm 0.34}$   & 0.69$_{\pm 0.17}$        & 0.44$_{\pm 0.08}$   & 0.43$_{\pm 0.07}$   & \textbf{1}$_{\pm 0.00}$      & 0.62$_{\pm 0.04}$        & 0.4$_{\pm 0.02}$    & 0.4$_{\pm 0.02}$    & \textbf{1}$_{\pm 0.00}$      & 0.6$_{\pm 0.01}$         & 0.4    & 0.39   & \textbf{1}      & 0.6         \\
GRACE     & \textbf{1}$_{\pm 0.00}$    & 0.36$_{\pm 0.02}$   & \textbf{1}$_{\pm 0.00}$      & 0.79$_{\pm 0.01}$        & \textbf{1}$_{\pm 0.00}$      & 0.15$_{\pm 0.04}$   & \textbf{1}$_{\pm 0.00}$      & 0.72$_{\pm 0.02}$        & \textbf{1}$_{\pm 0.00}$      & 0.15$_{\pm 0.01}$   & \textbf{1}$_{\pm 0.00}$      & 0.72$_{\pm 0.00}$        & \textbf{1}      & 0.02   & \textbf{1}      & 0.67        \\
SimIE & 0.62$_{\pm0.29}$ & 0.59$_{\pm0.29}$ & 0.94$_{\pm0.13}$ & 0.72$_{\pm0.14}$ & 0.61$_{\pm0.09}$& 0.58$_{\pm0.09}$ & 0.94$_{\pm0.04}$ &0.71$_{\pm0.05}$ & 0.60$_{\pm0.02}$ & 0.58$_{\pm0.02}$ & 0.93$_{\pm0.02}$ & 0.70$_{\pm0.01}$ & 0.59 & 0.56 & 0.92 & 0.69\\
WISE      & 0.96 $_{\pm 0.16}$  & 0.93  $_{\pm 0.21}$  & \textbf{1} $_{\pm 0.00}$      & \underline{0.96}    $_{\pm 0.09}$     & 0.88 $_{\pm 0.07}$    & \underline{0.84} $_{\pm 0.09}$   & \underline{0.99}    $_{\pm 0.00}$    & \underline{0.91} $_{\pm 0.04}$        & 0.82 $_{\pm 0.06}$   & 0.76$_{\pm 0.05}$     & \underline{0.99}   $_{\pm 0.00}$     & \underline{0.86}  $_{\pm 0.03}$       & 0.65    & 0.63   & \underline{0.99}     & \underline{0.76}       \\
AlphaEdit & 0.86$_{\pm 0.17}$ & 0.78 $_{\pm 0.22}$  & \underline{0.99}  $_{\pm 0.01}$      & 0.88 $_{\pm 0.09}$        & 0.82  $_{\pm 0.07}$  & 0.76  $_{\pm 0.08}$  & \underline{0.99}   $_{\pm 0.01}$   & 0.86  $_{\pm 0.04}$       & 0.80  $_{\pm 0.03}$   & 0.75   $_{\pm 0.04}$  & 0.96  $_{\pm 0.02}$ & 0.84 $_{\pm 0.02}$        & 0.80   & \underline{0.74}   & 0.75   & \underline{0.76}        \\
\midrule
\textbf{LOKI} (Ours)      & \textbf{1}$_{\pm 0.00}$    & \underline{0.96}$_{\pm 0.11}$& 0.98 $_{\pm 0.05}$& \textbf{0.98}  $_{\pm 0.06}$    & \underline{0.96}$_{\pm 0.04}$  & \textbf{0.91} $_{\pm 0.08}$ & 0.97 $_{\pm 0.02}$ & \textbf{0.95} $_{\pm 0.03}$     & \underline{0.95} $_{\pm 0.00}$ & \textbf{0.88}$_{\pm 0.03}$  & 0.93$_{\pm 0.02}$ & \textbf{0.92}$_{\pm 0.01}$ & \underline{0.92} & \textbf{0.86} & 0.82 & \textbf{0.87}\\
\bottomrule
\end{tabular}
}
\end{table*}

\begin{table*}[h!]
\caption{Performance comparison on the SelfCheckGPT hallucination correction task with standard deviations. Metrics include Reliability (Rel. $\downarrow$) and Locality (Loc. $\uparrow$) across $\tau$ sequential edits. For this task, Rel. is measured as perplexity (PPL) of the target text. Generalization is not available in this dataset since no rephrasing examples are provided. \textbf{Bold} indicates the best result and underlining indicates the second-best result for each metric. LOKI demonstrates superior Reliability throughout the editing process across both Llama-3-8B and Mistral-7B. Overall, it maintains competitive performance and avoids the reliability collapse observed in methods at higher edit scales.}
\centering
\scalebox{.74}{
\label{tab:hallucination_sd}
\begin{tabular}{l|cc|cc|cc|cc}
\toprule
Method & \multicolumn{2}{c|}{$\tau=1$} & \multicolumn{2}{c|}{$\tau=10$} & \multicolumn{2}{c|}{$\tau=100$} & \multicolumn{2}{c}{$\tau=600$} \\
 & Rel. ($\downarrow$) & Loc. ($\uparrow$) & Rel. ($\downarrow$) & Loc. ($\uparrow$) & Rel. ($\downarrow$) & Loc. ($\uparrow$) & Rel. ($\downarrow$) & Loc. ($\uparrow$) \\
\midrule
\multicolumn{9}{c}{Llama-3-8B-Instruct} \\
\midrule
FT        & \textbf{1.01}$_{\pm 0.00}$ & 0.95$_{\pm 0.07}$ & \underline{1.12}$_{\pm 0.10}$ & 0.84$_{\pm 0.04}$ & \underline{1.44}$_{\pm 0.12}$ & 0.65$_{\pm 0.02}$ & \underline{2.02} & 0.45 \\
ROME      & 1.39$_{\pm 1.05}$ & 0.97$_{\pm 0.04}$ & 4.74e1$_{\pm 175}$ & 0.61$_{\pm 0.09}$ & 5.00e4$_{\pm 7.07e4}$ & 0.02$_{\pm 0.01}$ & 2.89e4 & 0.01 \\
MEMIT     & 1.32$_{\pm 0.52}$ & \underline{0.99}$_{\pm 0.02}$ & 1.84e3$_{\pm 1.36e4}$ & 0.93$_{\pm 0.11}$ & 6.25e10$_{\pm 1.40e11}$ & 0.01$_{\pm 0.01}$ & 3.75e4 & 0.05 \\
DEFER     & 2.34e1$_{\pm 1.25e2}$ & 0.80$_{\pm 0.15}$ & 6.86e1$_{\pm 335}$ & \textbf{1.00}$_{\pm 0.01}$ & 7.12e1$_{\pm 95.9}$ & \textbf{1.00}$_{\pm 0.00}$ & 7.12e1 & \textbf{1.00} \\
GRACE     & 1.40$_{\pm 1.98}$ & \textbf{1.00}$_{\pm 0.00}$ & 7.10e1$_{\pm 335}$ & \textbf{1.00}$_{\pm 0.00}$ & 7.17e1$_{\pm 95.4}$ & \textbf{1.00}$_{\pm 0.00}$ & 7.54e1 & \textbf{1.00} \\
WISE      & 1.44$_{\pm 3.18}$ & 0.97$_{\pm 0.10}$ & 3.47$_{\pm 12.2}$ & 0.97$_{\pm 0.04}$ & 3.54$_{\pm 0.21}$ & \underline{0.99}$_{\pm 0.00}$ & 2.63e2 & \underline{0.99} \\
AlphaEdit & \underline{1.29}$_{\pm 1.04}$ & \textbf{1.00}$_{\pm 0.01}$ & 1.49$_{\pm 0.66}$ & \underline{0.99}$_{\pm 0.01}$ & 5.04$_{\pm 0.97}$ & 0.96$_{\pm 0.01}$ & 9.04e3 & 0.05 \\
\midrule
%\textbf{LOKI} (Ours) & \textbf{1.01}$_{\pm 0.00}$ & 0.98$_{\pm 0.02}$ & \textbf{1.01}$_{\pm 0.00}$ & 0.97$_{\pm 0.01}$ & \textbf{1.05}$_{\pm 0.02}$ & 0.89$_{\pm 0.03}$ & \textbf{1.16} & 0.78 \\
\textbf{LOKI} (Ours) & \textbf{1.01}$_{\pm 0.00}$ & 0.98$_{\pm 0.02}$ & \textbf{1.01}$_{\pm 0.01}$ & 0.98$_{\pm 0.01}$ & \textbf{1.03}$_{\pm 0.02}$ & 0.94$_{\pm 0.02}$ & \textbf{1.14} & 0.85 \\
\midrule
\multicolumn{9}{c}{Mistral-7B} \\
\midrule
FT        & \underline{1.02}$_{\pm 0.07}$ & 0.77$_{\pm 0.12}$ & \underline{1.41}$_{\pm 0.36}$ & 0.44$_{\pm 0.06}$ & 1.94$_{\pm 0.34}$ & 0.22$_{\pm 0.01}$ & \underline{7.26} & 0.20 \\
ROME      & 1.98$_{\pm 1.20}$ & \underline{0.99}$_{\pm 0.02}$ & 2.58$_{\pm 1.34}$ & 0.91$_{\pm 0.02}$ & 6.77e3$_{\pm 3.80e3}$ & 0.03$_{\pm 0.00}$ & 7.87e3 & 0.01 \\
MEMIT     & 1.59$_{\pm 1.15}$ & \textbf{1.00}$_{\pm 0.01}$ & 3.13e2$_{\pm 1.49e3}$ & 0.93$_{\pm 0.20}$ & 9.03e4$_{\pm 1.31e5}$ & 0$_{\pm 0.01}$ & 3.33e6 & 0.00 \\
DEFER     & 1.76e1$_{\pm 116}$ & 0.92$_{\pm 0.12}$ & 3.22e1$_{\pm 162}$ & \underline{0.99}$_{\pm 0.02}$ & 3.22e1$_{\pm 47.1}$ & \textbf{1.00}$_{\pm 0.00}$ & 3.22e1 & \textbf{1.00} \\
GRACE     & 1.92$_{\pm 2.04}$ & \textbf{1.00}$_{\pm 0.00}$ & 3.21e1$_{\pm 162}$ & \textbf{1.00}$_{\pm 0.00}$ & 3.22e1$_{\pm 47.2}$ & \textbf{1.00}$_{\pm 0.00}$ & 3.25e1 & \textbf{1.00} \\
WISE      & 1.08$_{\pm 0.66}$ & \textbf{1.00}$_{\pm 0.02}$ & 1.43e1$_{\pm 91.5}$ & 0.98$_{\pm 0.05}$ & \underline{1.74}$_{\pm 0.29}$ & 0.96$_{\pm 0.01}$ & 10.0 & \underline{0.92} \\
AlphaEdit & 1.71$_{\pm 1.50}$ & \textbf{1.00}$_{\pm 0.01}$ & 1.88$_{\pm 0.77}$ & \underline{0.99}$_{\pm 0.00}$ & 2.96$_{\pm 0.58}$ & \underline{0.98}$_{\pm 0.00}$ & 6.43e1 & 0.87 \\
\midrule
%\textbf{LOKI} (Ours) & \textbf{1.01}$_{\pm 0.00}$ & \underline{0.99}$_{\pm 0.02}$ & \textbf{1.02}$_{\pm 0.01}$ & 0.96$_{\pm 0.01}$ & \textbf{1.08}$_{\pm 0.01}$ & 0.87$_{\pm 0.01}$ & \textbf{1.43} & 0.72 \\
\textbf{LOKI} (Ours) & \textbf{1.01}$_{\pm 0.00}$ & \underline{0.99}$_{\pm 0.02}$ & \textbf{1.01}$_{\pm 0.01}$ & \underline{0.99}$_{\pm 0.01}$ & \textbf{1.04}$_{\pm 0.01}$ & 0.94$_{\pm 0.01}$ & \textbf{1.13} & 0.87 \\
\bottomrule
\end{tabular}
}
\end{table*}

\begin{table}[h!]
%\scriptsize
\caption{OOD generalization performance on the Temporal dataset (GPT-J-6B) with standard deviations. Metrics include Reliability (Rel. $\uparrow$), OOD Generalization (OOD Gen. $\uparrow$), and Locality (Loc. $\uparrow$) across $\tau$ sequential edits, with Avg. denoting their mean. \textbf{Bold} indicates the best result and underlining indicates the second-best result for each metric. LOKI maintains high Reliability and OOD Generalization across edit scales. While exhibiting lower Locality compared to specific memory-based methods, it prevents the significant reliability collapse observed in methods like ROME and MEMIT as $\tau$ increases.}
\centering
\scalebox{.80}{
\begin{tabular}{l|cccc|cccc}
\toprule
Method & \multicolumn{4}{c|}{$\tau = 10$} & \multicolumn{4}{c}{$\tau = 75$} \\
\cmidrule(lr){2-5} \cmidrule(lr){6-9}
& Rel. ($\uparrow$) & OOD Gen. ($\uparrow$) & Loc. ($\uparrow$) & Avg. ($\uparrow$) & Rel. ($\uparrow$) & OOD Gen. ($\uparrow$) & Loc. ($\uparrow$) & Avg. ($\uparrow$) \\
\midrule
w/o Editing & 0.56 & 0.21 & --- & --- & 0.56 & 0.21 & --- & --- \\
\midrule
FT & 0.96$_{\pm 0.02}$ & \underline{0.35}$_{\pm 0.03}$ & 0.69$_{\pm 0.02}$ & 0.67$_{\pm 0.01}$ & 0.93 & 0.32 & 0.62 & 0.62 \\
ROME & 0.10$_{\pm 0.04}$ & 0$_{\pm 0.00}$ & 0.06$_{\pm 0.02}$ & 0.05$_{\pm 0.01}$ & 0.02 & 0.00 & 0.02 & 0.01 \\
MEMIT & 0.82$_{\pm 0.05}$ & 0.28$_{\pm 0.02}$ & \underline{0.99}$_{\pm 0.01}$ & 0.70$_{\pm 0.02}$ & 0.43 & 0.14 & 0.49 & 0.35 \\
DEFER & 0.66$_{\pm 0.02}$ & 0.29$_{\pm 0.02}$ & 0.81$_{\pm 0.05}$ & 0.59$_{\pm 0.02}$ & 0.55 & 0.21 & 0.95 & 0.57 \\
GRACE & 0.59$_{\pm 0.02}$ & 0.22$_{\pm 0.02}$ & \textbf{1}$_{\pm 0.00}$ & 0.60$_{\pm 0.01}$ & 0.56 & 0.22 & \textbf{1} & 0.59 \\
WISE & \underline{0.99}$_{\pm 0.01}$ & \textbf{0.36}$_{\pm 0.02}$ & 0.93$_{\pm 0.06}$ & \textbf{0.76}$_{\pm 0.02}$ & \textbf{0.98} & \textbf{0.38} & \underline{0.99} & \textbf{0.78} \\
AlphaEdit & 0.90$_{\pm 0.03}$ & 0.27$_{\pm 0.02}$ & \underline{0.99}$_{\pm 0.01}$ & 0.72$_{\pm 0.01}$ & \underline{0.87} & 0.28 & 0.97 & 0.70 \\
%MEMOIR & 0.99 & 0.37 & 1 & 0.79 & 0.99 & 0.38 & 1 & 0.79 \\
\midrule
%\rowcolor[gray]{0.9} 
%\textbf{LOKI} (Ours) & \textbf{1}$_{\pm 0.00}$ & \textbf{0.36}$_{\pm 0.02}$ & 0.82$_{\pm 0.06}$ & \underline{0.72} & \textbf{0.99} & \underline{0.35} & 0.66 & 0.67 \\
\textbf{LOKI} (Ours) & \textbf{1}$_{\pm 0.00}$ & \textbf{0.36}$_{\pm 0.02}$ & 0.82$_{\pm 0.06}$ & \underline{0.73}$_{\pm 0.02}$ & \textbf{0.98} & \underline{0.36} & 0.81 & \underline{0.72} \\
\bottomrule
\end{tabular}
}
\label{tab:ood_generalization_sd}
\end{table}

\section{Combining AlphaEdit with LOKI layer selection}

\label{app:combining}

To further isolate the effect of our layer selection algorithm, we perform an experiment combining the LOKI dynamic layer selection algorithm, with the knowledge insertion algorithm of AlphaEdit. Meaning, for each sample, first the editing layers are selected using our HSIC-based dynamic layer selection as described in \cref{sec:layersel}, and then knowledge insertion is performed using the same method as used in AlphaEdit. The results are presented in \cref{tab:alphaloki}, and show that while across smaller $\tau$ the results are not much different, substantial improvements are observed for $\tau=1000$, the most challenging scenario. This is while both settings still underperform LOKI in most metrics. We hypothesize that while the layer selection does improve AlphaEdit, its gains are limited by AlphaEdits two step optimization/update scheme, as well as it's requirement to store features of previous edits for each layer in order to avoid catastrophic forgetting. LOKI does not suffer from these limitations.

\begin{table*}[t!]
%\caption{Results of editing two LLMs on the ZsRE dataset, with $\tau$ being the number of sequential edits. Avg represents the average of all three metrics. Higher is better for all metrics. LOKI outperforms existing methods consistently, especially so with a larger number of edits.}
\caption{Results on the ZsRE dataset (Llama-3-8B) for combining AlphaEdit with the layer selection algorithm of LOKI. Metrics include Reliability (Rel.~$\uparrow$), Generalization (Gen.~$\uparrow$), and Locality (Loc.~$\uparrow$) across $\tau$ sequential edits, with Avg. representing the mean of the three scores. \textbf{Bold} is the best result. Using the LOKI layer selection improves AlphaEdit performance substantially for $\tau=1000$.}% as $\tau$ increases up to 1,000.}
\centering
\label{tab:alphaloki}
\resizebox{0.90\linewidth}{!}{
\begin{tabular}{c|ccc|c|ccc|c|ccc|c|ccc|c}
\toprule
Method    & \multicolumn{4}{c|}{$\tau=$ 1}              & \multicolumn{4}{c|}{$\tau=$ 10}               & \multicolumn{4}{c|}{$\tau=$ 100}              & \multicolumn{4}{c}{$\tau=$ 1000}             \\
\midrule
          & Rel.  & Gen.    & Loc.    & Avg.         & Rel.    & Gen.    & Loc.    & Avg.         & Rel.    & Gen.    & Loc.    & Avg.         & Rel.    & Gen.    & Loc.    & Avg.         \\
          \midrule
% \multicolumn{17}{c}{Llama 3-8B-Instruct}                         \\
\midrule

AlphaEdit & \textbf{0.99} & \textbf{0.88}   & \textbf{0.99}      & \textbf{0.95}        & 0.92   & {0.84}   & \textbf{0.98}   & {0.92}        & {0.90 }  &{\textbf{0.84}}   & \textbf{0.94}   & \textbf{0.90}        & 0.76   & \textbf{0.88}   & 0.65   & {0.76}        \\
\midrule
AlphaEdit + LOKI layer selection      & \textbf{0.99}    & {0.86} & {\textbf{0.99}}  & \textbf{0.95} & \textbf{{0.94}} & \textbf{0.85} & {\textbf{0.98}} & \textbf{0.93} & \textbf{0.91} & {0.82}& {0.93} & {0.89}      & {\textbf{0.94}} & {0.80}  & {\textbf{0.78}} & \textbf{0.84} \\

\bottomrule
\end{tabular}
}
\end{table*}

\section{Limitations, Future Works and Broader Impacts}
\label{app:limit}

\paragraph{Limitations} LOKI attempts to solve the shortcomings of previous knowledge editing methods, such as requiring historical data or inadequate selection of editing parameters. While we achieve competitive performance without historical data and by performing dynamic per-sample layer selection, LOKI has its own limitations. Firstly, computing the SVD of large weight matrices is time-consuming. While LOKI is still faster than WISE and AlphaEdit, it still suffers from some computational overhead.
Second, while LOKI is the first to perform dynamic layer selection for knowledge editing, it is constrained to choosing entire layers for editing rather than individual parameters or neurons. Overall we believe the benefits of LOKI, especially bypassing the requirement of historical data and time-consuming pre-processing or additional parameters, outweigh its limitations. We believe LOKI is a first step towards elevating lifelong knowledge editing to new heights.

\paragraph{Future Works} LOKI takes an important step towards efficient and sustainable memory-free lifelong knowledge editing. However, there is still room for improvement. Firstly, as mentioned in the limitations section, one can devise an algorithm to dynamically select individual parameters or neurons rather than entire layers, giving more flexibility to the algorithm for more precise knowledge insertion. Second, the weight null-space can be further constrained using heuristics or additional theorems to provide further decoupling from past knowledge, improving locality.

\paragraph{Broader Impacts} In this work, we present LOKI, a novel knowledge editing method that aims to facilitate the incorporation of new knowledge into LLMs. As LLMs see widespread use in society, knowledge editing in LLMs can be beneficial in a variety of aspects: fixing model mistakes~\cite{balachandran2022correcting}, correcting hallucinations~\cite{ji2023survey}, or removing bias~\cite{ferrara2023should}. However, as with any algorithm capable of modifying LLM behaviors, knowledge editing methods can be misused. For example, models can be edited to encourage toxic behavior, generate misinformation, or use inappropriate language. These risks are not exclusive to LOKI, and we believe the behavior and usage of LLMs should be sufficiently regulated to prevent misuse, including with knowledge editing. Overall, we believe that the benefits of LOKI are substantial enough to outweigh potential risks and that LOKI does not raise any new or exclusive ethical concerns.

%%%%%%%%%%%%%%%%%%%%%%%%%%%%%%%%%%%%%%%%%%%%%%%%%%%%%%%%%%%%

% \newpage
% \input{checklist.tex}

\end{document}